\def\year{2022}\relax
\documentclass[letterpaper]{article} 
\usepackage{aaai22}  
\usepackage{times}  
\usepackage{helvet}  
\usepackage{courier}  
\usepackage[hyphens]{url}  
\usepackage{graphicx} 
\usepackage{natbib}  
\usepackage{caption} 
\DeclareCaptionStyle{ruled}{labelfont=normalfont,labelsep=colon,strut=off} 
\usepackage{algorithm}

\usepackage{newfloat}
\usepackage{listings}
\floatstyle{ruled}
\newfloat{listing}{tb}{lst}{}
\floatname{listing}{Listing}
\usepackage{amsmath,amsfonts,bm}

\def\figref#1{figure~\ref{#1}}

\def\eqref#1{equation~\ref{#1}}

\def\algref#1{algorithm~\ref{#1}}

\def\1{\bm{1}}

\def\rva{{\mathbf{a}}}
\def\rvb{{\mathbf{b}}}
\def\rvc{{\mathbf{c}}}

\def\rvh{{\mathbf{h}}}

\def\rvp{{\mathbf{p}}}

\def\rvw{{\mathbf{w}}}
\def\rvx{{\mathbf{x}}}
\def\rvy{{\mathbf{y}}}

\def\rmD{{\mathbf{D}}}

\def\rmI{{\mathbf{I}}}
\def\rmJ{{\mathbf{J}}}

\def\rmW{{\mathbf{W}}}

\def\vs{{\bm{s}}}

\DeclareMathAlphabet{\mathsfit}{\encodingdefault}{\sfdefault}{m}{sl}
\SetMathAlphabet{\mathsfit}{bold}{\encodingdefault}{\sfdefault}{bx}{n}

\newcommand{\R}{\mathbb{R}}

\newcommand{\softmax}{\mathrm{softmax}}

\DeclareMathOperator*{\argmax}{argmax}

\DeclareMathOperator{\sign}{sign}

\usepackage{url}

\usepackage{multirow}
\usepackage{booktabs}
\usepackage{subcaption}

\newcommand{\ignore}[1]{}
\newcommand{\norm}[1]{\left\lVert#1\right\rVert}

\usepackage[noend]{algpseudocode}
\usepackage{algorithm}

\usepackage{amsthm}
\theoremstyle{definition}
\newtheorem{pro}{Proposition}
\theoremstyle{plain}

\newcommand{\SKIP}[1]{}

\newcommand{\bfx}{\mathbf{x}}

\newcommand{\bfy}{\mathbf{y}}

\newcommand{\I}{\mathbbm{1}}
\newcommand{\calX}{\mathcal{X}}

\newcommand{\bfzero}{\boldsymbol{0}}

\renewcommand{\R}{\rm I\!R}

\newcommand{\bfp}{\mathbf{p}}

\newcommand{\calD}{\mathcal{D}}

\newcommand{\bfw}{\mathbf{w}}

\newcommand{\bfu}{\mathbf{u}}

\newcommand{\bfa}{\mathbf{a}}

\newcommand{\calQ}{\mathcal{Q}}

\renewcommand{\sign}{\operatorname{sign}}

\renewcommand{\softmax}{\operatorname{softmax}}

\newcommand{\bfg}{\mathbf{g}}

\newcommand{\sbfw}{\mathbf{w}^*}

\newcommand{\alllayers}{\{1\ldots K\}}

\newcommand{\bfv}{\mathbf{v}}

\newcommand{\bnns}{\acrshort{BNN}s}

\def\myref#1{{\color{red}{#1}}}%

\usepackage{xcolor}

\usepackage{bbm}

\graphicspath{{images/}}

\usepackage[page]{appendix}
\usepackage{chngcntr}
\usepackage{etoolbox}

\AtBeginEnvironment{subappendices}{%
\section*{Appendix}
\addcontentsline{toc}{section}{Appendix}
}

\usepackage{xspace}
\makeatletter
\DeclareRobustCommand\onedot{\futurelet\@let@token\@onedot}
\def\@onedot{\ifx\@let@token.\else.\null\fi\xspace}

\def\eg{\emph{e.g}\onedot} 
\def\ie{\emph{i.e}\onedot} 
 
 \def\vs{\emph{vs}\onedot}

\makeatother

\usepackage{enumitem}
\newenvironment{tight_itemize}{
\begin{itemize}[leftmargin=10pt]
  \setlength{\topsep}{0pt}
  \setlength{\itemsep}{0pt}
  \setlength{\parskip}{0pt}
  \setlength{\parsep}{0pt}
}{\end{itemize}}
\newenvironment{tight_enumerate}{
\begin{enumerate}[leftmargin=10pt]
  \setlength{\topsep}{0pt}
  \setlength{\itemsep}{0pt}
  \setlength{\parskip}{0pt}
  \setlength{\parsep}{0pt}
}{\end{enumerate}}

\def\figref#1{Fig.~\ref{#1}}

\def\eqref#1{Eq.~(\ref{#1})}

\def\algref#1{Algorithm~\ref{#1}}

\def\tabref#1{Table~\ref{#1}}

\def\proref#1{Proposition~\ref{#1}}

\def\tabref#1{Table~\ref{#1}}

\algnewcommand\INPUT{\item[\textbf{Input:}]}%
\algnewcommand\OUTPUT{\item[\textbf{Output:}]}%

\usepackage[acronym,smallcaps]{glossaries} 
\glsdisablehyper
\newacronym{MRF}{mrf}{Markov Random Field}
\newacronym{SGD}{sgd}{Stochastic Gradient Descent}
\newacronym{GD}{gd}{Gradient Descent}
\newacronym{PGD}{pgd}{Projected Gradient Descent}
\newacronym{PGD++}{pgd++}{Improved Projected Gradient Descent}
\newacronym{FGSM}{fgsm}{Fast Gradient Sign Method}
\newacronym{FGSM++}{fgsm++}{Improved Fast Gradient Sign Method}
\newacronym{IFGSM}{ifgsm}{Iterative Fast Gradient Sign Method}
\newacronym{TS}{ts}{Temperature Scaling}
\newacronym{JSV}{jsv}{Jacobian Singular Values}
\newacronym{MJSV}{mjsv}{Mean Jacobian Singular Values}
\newacronym{ICM}{icm}{Iterative Conditional Modes}
\newacronym{DNN}{dnn}{Deep Neural Networks}
\newacronym{NN}{nn}{Neural Network}
\newacronym{PMF}{pmf}{Proximal Mean-Field}
\newacronym{PICM}{picm}{Proximal Iterative Conditional Modes}
\newacronym{BC}{bc}{BinaryConnect}
\newacronym{BWN}{bwn}{Binary Weight Network}
\newacronym{IP}{ip}{Integer Programming}
\newacronym{fc}{fc}{fully-connected}
\newacronym{REF}{ref}{Reference Network}
\newacronym{KL}{kl}{KL}
\newacronym{LR}{lr}{LR}
\newacronym{KKT}{kkt}{KKT}
\newacronym{MAP}{map}{Maximum a Posteriori}
\newacronym{MDA}{mda}{Mirror Descent Algorithm}
\newacronym{MD}{md}{Mirror Descent}
\newacronym{EDA}{eda}{Entropic Descent Algorithm}
\newacronym{ED}{ed}{Entropic Descent}
\newacronym{EGD}{egd}{Exponentiated Gradient Descent}
\newacronym{PQ}{pq}{ProxQuant}
\newacronym{STE}{ste}{Straight Through Estimator}
\newacronym{HGD}{hgd}{Hybrid Gradient Descent}
\newacronym{S}{s}{S}
\newacronym{LLM}{llm}{Local Linearity Measure}
\newacronym{BNN-WQ}{bnn-wq}{Binary Neural Network with Weights Quantized}
\newacronym{BNN-WAQ}{bnn-waq}{Binary Neural Network with Weights and Activations Quantized}
\newacronym{BBA}{bba}{Brendel \& Bethge Attack}
\newacronym{CLEVER}{clever}{CLEVER}
\newacronym{BNN}{bnn}{Binary Neural Network}
\newacronym{JSVS-GT}{jsvs-gt}{Jacobian Singular Values Scaling and Gradient Thresholding}
\newacronym{NJS}{njs}{Network Jacobian Scaling}
\newacronym{HNS}{hns}{Hessian Norm Scaling}
\newacronym{IGNS}{igns}{Input Gradient Norm Scaling}
\newacronym{ORIG.}{orig.}{Original}
\newacronym{L-BFGS}{l-bfgs}{L-BFGS}
\newacronym{CW}{cw}{Carlini Wagner}
\newacronym{APGD}{apgd}{Auto-\acrshort{PGD}}
\newacronym{DLR}{dlr}{Difference of Logits Ratio}
\newacronym{IPROP}{iprop}{IPROP}

\newcommand{\svgg}[1]{{\small VGG#1}}
\newcommand{\sresnet}[1]{{\small ResNet#1}}
\newcommand{\sdensenet}[1]{{\small DenseNet#1}}

\newcommand{\cifar}[1]{{\small CIFAR#1}}

\usepackage{pifont}

\usepackage{wrapfig}

\newif\ifsupp
\suppfalse

\newif\ifarxiv
\arxivfalse

\newif\iffinal
\finalfalse

\newcommand{\citenew}[1]{\cite{#1}}

\title{Improved Gradient based Adversarial Attacks for Quantized Networks}

\author{
    Kartik Gupta \textsuperscript{\rm1,3}, Thalaiyasingam Ajanthan  $^\dagger$\textsuperscript{\rm1,2} 
}
\affiliations{
    \textsuperscript{\rm 1}Australian National University
    \textsuperscript{\rm 2}Amazon Science 
    \textsuperscript{\rm 3}DATA61, CSIRO \\
    $^\dagger$ Work done prior to joining Amazon 

}

\begin{document}

\maketitle
\begin{abstract}
Neural network quantization has become increasingly popular due to efficient memory consumption and faster computation resulting from bitwise operations on the quantized networks. 
Even though they exhibit excellent generalization capabilities, their robustness properties are not well-understood.
In this work, we systematically study the robustness of quantized networks against gradient based adversarial attacks and demonstrate that these quantized models suffer from gradient vanishing issues and show a fake sense of robustness.
By attributing gradient vanishing to poor forward-backward signal propagation in the trained network, we introduce a simple temperature scaling approach to mitigate this issue while preserving the decision boundary.
Despite being a simple modification to existing gradient based adversarial attacks, experiments on multiple image classification datasets with multiple network architectures demonstrate that our temperature scaled attacks obtain near-perfect success rate on quantized networks while outperforming original attacks on adversarially trained models as well as floating-point networks\footnote{Open-source implementation available at \url{https://github.com/kartikgupta-at-anu/attack-bnn}}.
\end{abstract}

\section{Introduction}
\acrfull{NN} quantization has become increasingly popular due to reduced memory and time complexity enabling real-time applications and inference on resource-limited devices. 
Such quantized networks often exhibit excellent generalization capabilities despite having low capacity due to reduced precision for parameters and activations. 
However, their robustness properties are not well-understood. 
In particular, while parameter quantized networks are claimed to have better robustness against gradient based adversarial attacks~\citenew{galloway2018attacking}, activation only quantized methods are shown to be vulnerable~\citenew{lin2018defensive}.

In this work, we consider the extreme case of Binary Neural Networks~(\acrshort{BNN}s) and systematically study the robustness properties of parameter quantized models, as well as both parameter and activation quantized models against gradient based adversarial attacks.
Our analysis reveals that these quantized models suffer from gradient masking issues~\citenew{athalye2018obfuscated} and in turn show fake robustness.
We attribute this vanishing gradients issue to poor forward-backward signal propagation caused by trained binary weights, and our idea is to improve signal propagation of the network without changing the prediction.

There is a body of work on improving signal propagation in a neural network (\eg,~\cite{glorot2010understanding,pennington2017resurrecting,lu2020bidirectional}), however, we are facing a unique challenge of {\em improving signal propagation while preserving the decision boundary}, since our ultimate objective is to generate adversarial attacks. 
To this end, we first discuss the conditions to ensure informative gradients and then resort to a temperature scaling approach~\citenew{guo2017calibration} (which scales the logits before applying softmax cross-entropy) to show that, even with a single positive scalar the vanishing gradients issue in \bnns{} can be alleviated achieving {\it near perfect success rate}.

Specifically, we introduce two techniques to choose the temperature scale: 1) based on the singular values of the input-output Jacobian, 2) by maximizing the norm of the Hessian of the loss with respect to the input.
The justification for the first case is that if the singular values of input-output Jacobian are concentrated around $1$ (defined as dynamical isometry~\citenew{pennington2017resurrecting}) then the network is said to have good signal propagation. 
On the other hand, the intuition for maximizing the Hessian norm is that if the Hessian norm is large, then the gradient of the loss with respect to the input is sensitive to an infinitesimal change in the input. 
This is a sufficient condition for the network to have good signal propagation as well as informative gradients under the assumption that the network does not have any randomized or non-differentiable components.

In summary, this paper makes the following contributions:
\vspace{-2.5ex}
\begin{itemize}[leftmargin=*]
    \item We first show via various empirical checks that \bnns{} possess fake robustness against gradient based adversarial attacks such as \acrshort{FGSM}~\citenew{goodfellow2014explaining} and \acrshort{PGD}~\citenew{madry2017towards}.
    \item By accounting poor signal propagation for the failure of gradient based adversarial attacks, we present temperature scaling based solution to improve the existing attacks without changing the prediction of the network.
    \item In order to estimate appropriate scalar for temperature scaling in gradient based adversarial attacks, we present two variants namely \acrfull{NJS} and \acrfull{HNS} motivated from point of view of improving the signal propagation.
    \item With experimental evaluations using several network architectures on \cifar{-10/100} datasets, we show that our proposed techniques to modify existing gradient based adversarial attacks achieve near perfect success rate on \bnns{} with weight quantized (\acrshort{BNN-WQ}) and weight and activation quantized (\acrshort{BNN-WAQ}). Furthermore, our variants improves attack success even on adversarially trained models as well as floating point networks, showing the significance of signal propagation for adversarial attacks.
\end{itemize}
\section{Preliminaries}
We first provide some background on the neural network quantization and adversarial attacks. 

\subsection{Neural Network Quantization}\label{sec:nnq}
\gls{NN} quantization is defined as training networks with parameters constrained to a minimal, discrete set of quantization levels. 
This primarily relies on the hypothesis that since \gls{NN}s are usually overparametrized, it is possible to obtain a quantized network with performance comparable to the floating point network.
Given a dataset $\calD=\{\bfx_i, \bfy_i\}_{i=1}^n$, \gls{NN} quantization can be written as:
\vspace{-1ex}
\begin{equation}\label{eq:dnnobj}
\min_{\bfw\in \calQ^m} L(\bfw;\calD) := \frac{1}{n} \sum_{i=1}^n
\ell(\bfw;(\bfx_i,\bfy_i))\ .
\end{equation}
Here, $\ell(\cdot)$ denotes the input-output mapping composed with a standard loss function (\eg, cross-entropy loss), $\bfw$ is the $m$ dimensional parameter vector, and $\calQ$ is a predefined discrete set representing quantization levels (\eg, $\calQ=\{-1,1\}$ in the binary case). 

Most of the \gls{NN} quantization approaches~\citenew{ajanthan2019proximal,ajanthan2019mirror,bai2018proxquant,hubara2017quantized} convert the above problem into an unconstrained problem by introducing auxiliary variables
and optimize via (stochastic) gradient descent. 
To this end, the algorithms differ in the choice of quantization set (\eg, keep it discrete~\citenew{courbariaux2015binaryconnect}, relax it to the convex hull~\citenew{bai2018proxquant} or convert the problem into a lifted probability space~\citenew{ajanthan2019proximal}), the projection used, and how differentiation through projection is performed. 
In the case when the constraint set is relaxed, a gradually increasing annealing hyperparameter is used to enforce a quantized solution~\citenew{ajanthan2019proximal,ajanthan2019mirror,bai2018proxquant}. We refer the interested reader to respective papers for more detail. In this paper, we use \acrshort{BNN-WQ} obtained using \acrshort{MD}-$\tanh$-\acrshort{S}~\citenew{ajanthan2019mirror} and \acrshort{BNN-WAQ} obtained using obtained using Straight Through Estimation ~\cite{hubara2017quantized}. Briefly \acrshort{MD}-$\tanh$-\acrshort{S} represents network binarization method based on mirror descent optimization where the mirror map is derived using $\tanh$ projection function.

\subsection{Adversarial Attacks}\label{sec:adv_attack}
Adversarial examples consist of imperceptible perturbations to the data that alter the model's prediction with high confidence.
Existing attacks can be categorized into white-box and black-box attacks where the difference lies in the knowledge of the adversaries. White-box attacks allow the adversaries access to the target model's architecture and parameters, whereas black-box attacks can only query the model.
Since white-box gradient based attacks are popular, we summarize them below.

First-order gradient based attacks can be compactly written as \gls{PGD} on the negative of the loss function~\citenew{madry2017towards}. 
Formally, let $\bfx^0\in\R^N$ be the input image, then at iteration $t$, the \gls{PGD} update can be written as:
\begin{equation}\label{eq:pgd-general}
\bfx^{t+1} = P\left(\bfx^{t} + \eta\, \bfg^t_{\bfx}\right)\ ,
\end{equation}
where $P:\R^N \to \calX$ is a projection, $\calX\subset\R^N$ is the constraint set that bounds the perturbations, $\eta >0$ is the step size, and $\bfg^t_{\bfx}$ is a form of gradient of the loss with respect to the input $\bfx$ evaluated at $\bfx^t$.
With this general form, the popular gradient based adversarial attacks can be specified:
\begin{tight_itemize}
\item {\bf \gls{FGSM}}: This is a one step attack introduced in~\cite{goodfellow2014explaining}. Here, $P$ is the identity mapping, $\eta$ is the maximum allowed perturbation magnitude, and $\bfg_{\bfx}^t = \sign\left(\nabla_{\bfx} \ell(\sbfw; (\bfx^t, \bfy))\right)$, where $\ell$ denotes the loss function, $\sbfw$ is the trained weights and $\bfy$ is the ground truth label corresponding to the image $\bfx^0$. 
\item {\bf \gls{PGD} with $L_{\infty}$ bound}: Arguably the most popular adversarial attack introduced in~\cite{madry2017towards} and sometimes referred to as \gls{IFGSM}. Here, $P$ is the $L_{\infty}$ norm based projection, $\eta$ is a chosen step size, and $\bfg_{\bfx}^t = \sign\left(\nabla_{\bfx} \ell(\sbfw; (\bfx^t, \bfy))\right)$, the sign of gradient same as \gls{FGSM}.
\item {\bf \gls{PGD} with $L_{2}$ bound}: This is also introduced in~\cite{madry2017towards} which performs the standard \gls{PGD} in the Euclidean space. Here, $P$ is the $L_{2}$ norm based projection, $\eta$ is a chosen step size, and $\bfg_{\bfx}^t = \nabla_{\bfx} \ell(\sbfw; (\bfx^t, \bfy))$ is simply the gradient of the loss with respect to the input.
\end{tight_itemize}
These attacks have been further strengthened by a random initial step~\citenew{tramer2017ensemble}.
In this paper, we perform this single random initialization for all experiments with \acrshort{FGSM}/\acrshort{PGD} attack unless otherwise mentioned.

\begin{figure*}[!t]
        \begin{subfigure}{0.33\textwidth}
          \includegraphics[width=\textwidth]{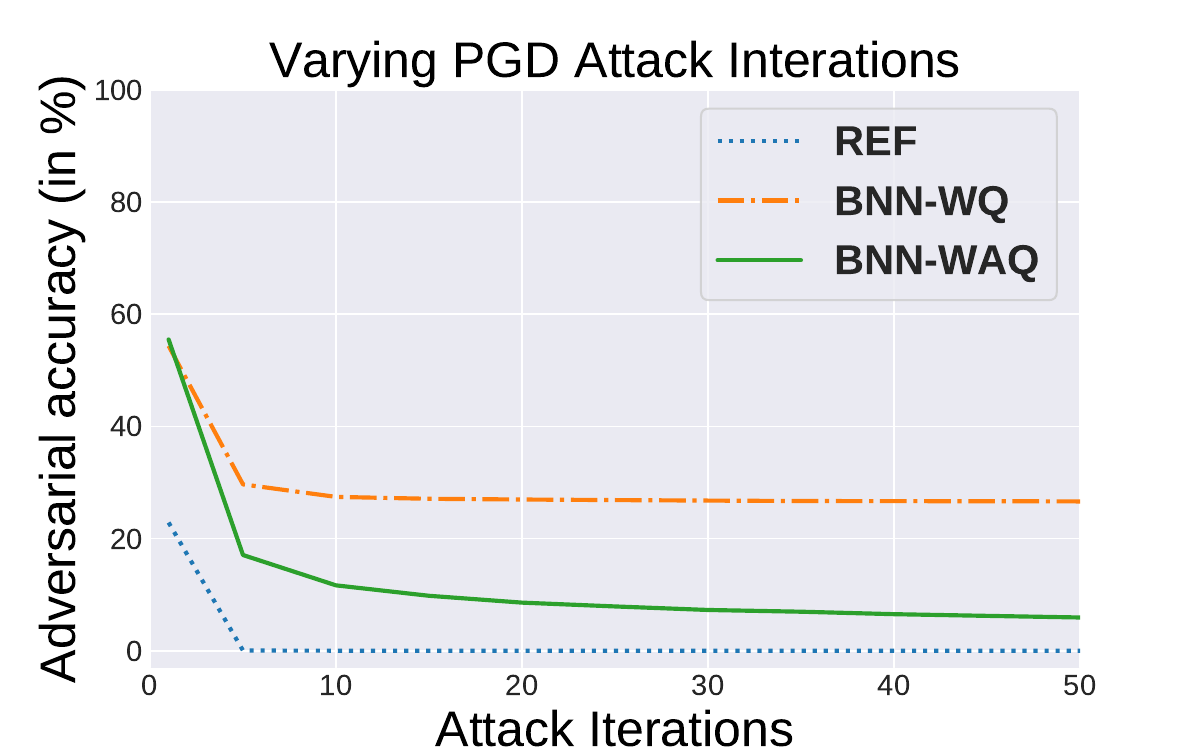}
          \vspace{-2ex}
        \caption{Attack iterations does not improve attack.}          
        \end{subfigure}          
        \begin{subfigure}{0.33\textwidth}
          \includegraphics[width=\textwidth]{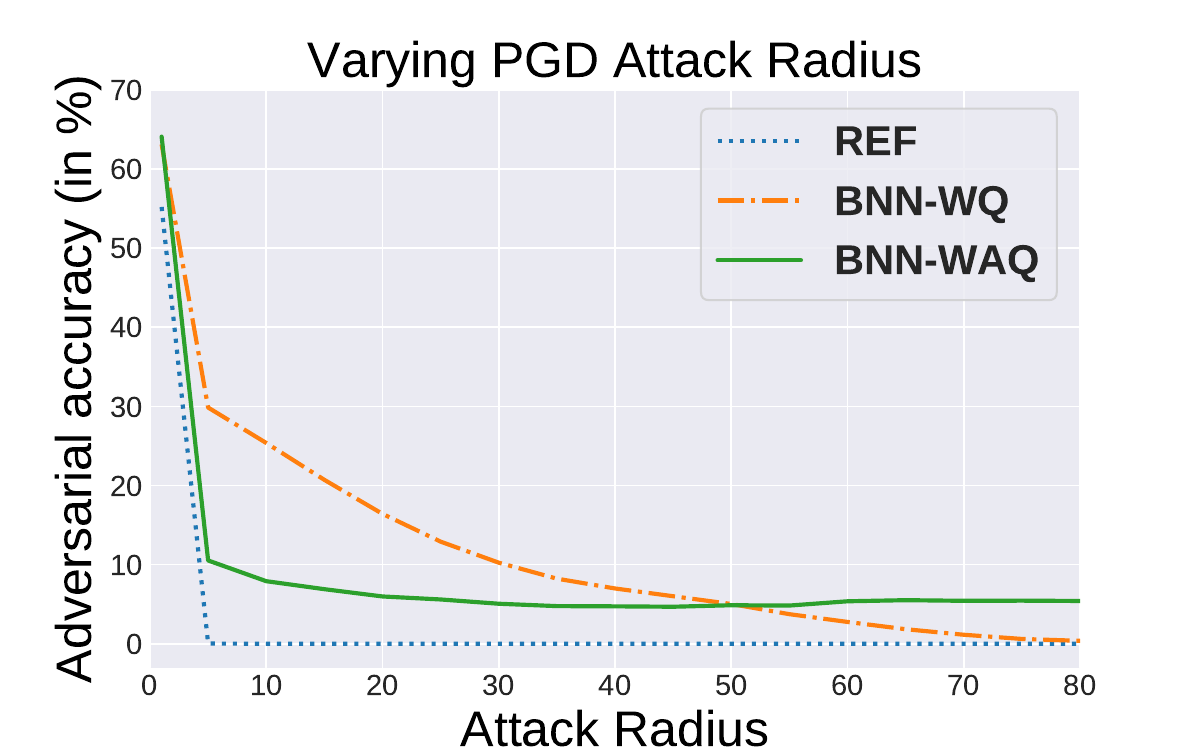}
          \vspace{-2ex}
        \caption{Attack radius does not improve attack.}          
        \end{subfigure}          
        \begin{subfigure}{0.28\textwidth}
          \includegraphics[width=\textwidth]{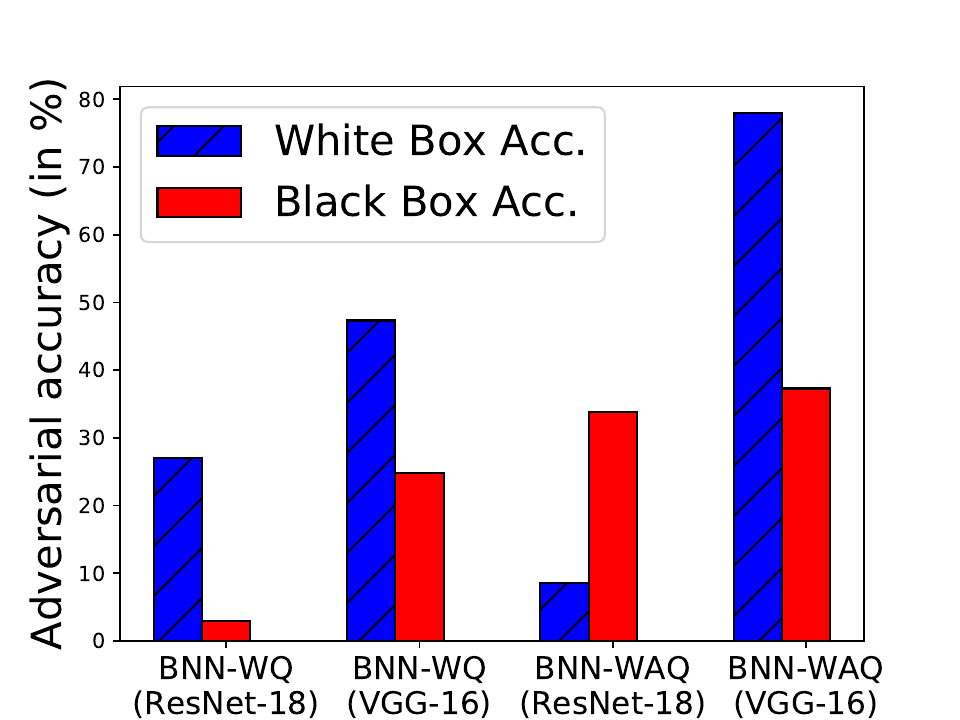}
          \vspace{-2ex}
        \caption{Black-box attacks perform better.}          
        \end{subfigure}          
    \vspace{-2ex}
    \caption{\em Gradient masking checks in \sresnet{-18} on \cifar{-10} for \acrshort{PGD} attack with $L_\infty$ bound. While (a), (c) show signs of gradient masking, (b) does not. We attribute this discrepancy to the random initial step before \gls{PGD}.
    }
\label{fig:obfuscated}
\end{figure*}

\begin{table}[!t]
    \scriptsize
    \centering
    \begin{tabular}{l|c@{\hspace{4pt}}c@{\hspace{4pt}}c|c@{\hspace{4pt}}c@{\hspace{4pt}}c}
        \toprule
         \multirow{2}{*}{\textbf{Method}} & \multicolumn{3}{c|}{\textbf{ResNet-18}} & \multicolumn{3}{c}{\textbf{VGG-16}} \\
         \cmidrule{2-7}
         & \textbf{Clean} & \textbf{Adv.(1)} & \textbf{Adv.(20)} & \textbf{Clean} & \textbf{Adv.(1)} & \textbf{Adv.(20)}\\
        \midrule
        \textbf{\acrshort{REF}} & 94.46 & 0.00 & 0.00 & 93.31 & 0.04 & 0.00 \\
        \textbf{\acrshort{BNN-WQ}} & 93.18 & 26.98 & 17.91 & 91.53 & 47.32 & 38.49 \\
        \textbf{\acrshort{BNN-WAQ}} & 87.67 & 8.57 & 1.94 & 89.69 & 78.01 & 59.26\\
        \bottomrule
    \end{tabular}
    \caption{\em Clean and adversarial accuracy (PGD attack with $L_\infty$ bound) on the test set of \cifar{-10} using \sresnet{-18} and \svgg{-16}. In brackets, we mention number of random restarts used to perform the attack. Note, \acrshort{BNN}s yield higher adversarial accuracy than floating point networks consistently.}
    \label{tab:adv_rob_bnn_cifar}
\end{table}

\SKIP{
\begin{wraptable}{r}{7.5cm}
    \scriptsize
    \centering
    \vspace{-2.5ex}    
    \begin{tabular}{l|c@{\hspace{4pt}}c@{\hspace{4pt}}c|c@{\hspace{4pt}}c@{\hspace{4pt}}c}
        \toprule
         \multirow{2}{*}{\textbf{Method}} & \multicolumn{3}{c|}{\textbf{ResNet-18}} & \multicolumn{3}{c}{\textbf{VGG-16}} \\
         \cmidrule{2-7}
         & \textbf{Clean} & \textbf{Adv.(1)} & \textbf{Adv.(20)} & \textbf{Clean} & \textbf{Adv.(1)} & \textbf{Adv.(20)}\\
        \midrule
        \textbf{\acrshort{REF}} & 94.46 & 0.00 & 0.00 & 93.31 & 0.04 & 0.00 \\
        \textbf{\acrshort{BNN-WQ}} & 93.18 & 26.98 & 17.91 & 91.53 & 47.32 & 38.49 \\
        \textbf{\acrshort{BNN-WAQ}} & 87.67 & 8.57 & 1.94 & 89.69 & 78.01 & 59.26\\
        \bottomrule
    \end{tabular}
    \vspace{-1ex}
    \caption{\em Clean and adversarial accuracy (PGD attack with $L_\infty$ bound) on the test set of \cifar{-10} using \sresnet{-18} and \svgg{-16}. In brackets, we mention number of random restarts used to perform the attack. Note, \acrshort{BNN}s yield higher adversarial accuracy than floating point networks consistently.}
    \vspace{-4.5ex}    
    \label{tab:adv_rob_bnn_cifar}

\end{wraptable}
}

\section{Robustness Evaluation of \bnns{}}\label{sec:robustness}
We start by evaluating the adversarial accuracy (i.e.\ accuracy on the perturbed data) of \acrshort{BNN}s using the \gls{PGD} attack with 
 perturbation bound of $8$ pixels (assuming each pixel in the image is in $[0,255]$) with respect to $L_{\infty}$ norm, step size $\eta = 2$ and the total number of iterations $T=20$. 
The attack details are the same in all evaluated settings unless stated otherwise.
We perform experiments on \cifar{-10} dataset using \sresnet{-18} and \svgg{-16} architectures and report the clean accuracy and \acrshort{PGD} adversarial accuracy with 1 and 20 random restarts in \tabref{tab:adv_rob_bnn_cifar}. 
It can be clearly and consistently observed that binary networks have high adversarial accuracy compared to the floating point counterparts. Even with 20 random restarts, \bnns{} clearly outperform floating point networks in terms of adversarial accuracy. 
Since this result is surprising, we investigate this phenomenon further to understand whether \acrshort{BNN}s are actually robust to adversarial perturbations or they show a fake sense of security due to obfuscated gradients~\citenew{athalye2018obfuscated}.


\paragraph{Identifying Obfuscated Gradients.} Recently, it has been shown that several defense mechanisms intentionally or unintentionally break gradient descent and cause obfuscated gradients and thus exhibit a false sense of security~\citenew{athalye2018obfuscated}. Several gradient based adversarial attacks tend to fail to produce adversarial perturbations in scenarios where the gradients are uninformative, referred to as gradient masking. Gradient masking can occur due to shattered gradients, stochastic gradients or exploding and vanishing gradients. 
We try to identify gradient masking in binary networks based on the empirical checks provided in~\cite{athalye2018obfuscated}. If any of these checks fail, it indicates gradient masking issue in \bnns. 

To illustrate this, we analyse the effects of varying different hyperparameters of \gls{PGD} attack on \acrshort{BNN}s trained on \cifar{-10} using \sresnet{-18} architecture.
Even though varying \gls{PGD} perturbation bound does not show any signs of gradient masking, varying attack iterations and black-box vs white-box results (on \sresnet{-18} and \svgg{-16}) clearly indicate gradient masking issues as depicted in \figref{fig:obfuscated}. The black-box attack outperforming white-box attack for \bnns{} certainly indicates gradient masking issues since the black-box attack do not use the gradient information from model being attacked.  
Here, our black-box model to a \acrshort{BNN} is the analogous floating point network trained on the same dataset and the attack is the same \gls{PGD} with $L_{\infty}$ bound.

\textit{These checks demonstrate that \acrshort{BNN}s are prone to gradient masking and exhibit fake robustness.} 
Note, shattered gradients occur due to non-differentiable components in the defense mechanism and stochastic gradients are caused by randomized gradients.
Since \acrshort{BNN}s are trainable from scratch and does not have randomized gradients, we narrow down gradient masking issue to vanishing or exploding gradients.
Since, vanishing or exploding gradients occur due to poor signal propagation, by introducing a single scalar, we discuss two approaches to mitigate this issue, which lead to almost $100\%$ success rate for gradient based attacks on \acrshort{BNN}s.


\section{Signal Propagation of Neural Networks}
We first describe how poor signal propagation in neural networks can cause vanishing or exploding gradients. 
Then we discuss the idea of introducing a single scalar to improve the existing gradient based attacks without affecting the prediction (\ie, decision boundary) of the trained models.
\SKIP{
For notational convenience, similar to~\cite{pennington2017resurrecting}, we consider a fully-connected neural network $f_\rvw$ with weights ${\rmW^l\in\R^{N_l\times N_{l-1}}}$, biases $\rvb^l\in\R^{N_{l-1}}$, pre-activations $\rvh^l\in\R^{N_l}$, and post-activations $\rva^l\in\R^{N_l}$, for $l \in \alllayers$ up to $K$ layers. Now, the feed-forward dynamics can be formulated as,
\begin{equation}
\label{eq:forward}
\rva^l = \phi(\rvh^l)\ ,\qquad \rvh^l = \rmW^l\rva^{l-1} + \rvb^l\ ,
\end{equation}
where $\phi: \R \to\R$ is an elementwise nonlinearity, and the input is denoted by $\rva^0=\rvx^0\in\R^N$.}

We consider a neural network $f_\rvw$ for an input $\rvx^0$, having post-activations $\rva^l$, for $l \in \alllayers$ up to $K$ layers and logits $\rva^K=f_\rvw(\rvx^0)$.
\SKIP{
Now consider the input-output Jacobian $\rmJ \in \mathbb{R}^{N\times N}$ given by
\begin{equation}
\begin{split}
\label{eqn:Jz}
\rmJ = \frac{\partial \rva^K}{\partial \rvx^0} = \prod_{l=1}^K \rmD^l \rmW^l.
\end{split}
\end{equation}
Here $\rmD^l$ is a diagonal matrix with entries $D^l_{ij} = \phi'({h}^l_i) \, \delta_{ij}$, where $\delta_{ij} = \I[i=j]$ is the Kronencker delta function and $\phi'$ denotes the derivative of non-linearity $\phi$. 
}
Now, since softmax cross-entropy is usually used as the loss function, we can write:
\begin{equation}\label{eq:smce}
    \ell(\bfa^K, \bfy) = -\bfy^T \log(\bfp)\ ,\qquad \bfp = \softmax(\bfa^K)\ ,
\end{equation}
where $\bfy\in\R^d$ is the one-hot encoded target label and $\log$ is applied elementwise.

For various gradient based adversarial attacks discussed earlier, gradient of the loss $\ell$ is used with respect to the input $\rvx^0$, which can also be formulated using chain rule as,
\begin{equation}
\begin{split}
\label{eqn:grad_input}
\frac{\partial \ell(\bfa^K, \bfy)}{\partial \rvx^0} = \frac{\partial \ell(\bfa^K, \bfy)}{\partial \rva^K}\frac{\partial \rva^K}{\partial \rvx^0} = \psi(\rva^K, \rvy)\, \rmJ \ ,
\end{split}
\end{equation}
where $\psi$ denotes the error signal and $\rmJ\in \mathbb{R}^{d\times N}$ is the input-output Jacobian. 
Here we use the convention that $\partial \bfv/\partial \bfu$ is of the form $\bfv$-size $\times$ $\bfu$-size.

Notice there are two components that influence the gradients, 1) the Jacobian $\rmJ$ and 2) the error signal $\psi$.
Gradient based attacks would fail if either the Jacobian is poorly conditioned or the error signal has saturating gradients, both of these will lead to vanishing gradients in $\partial \ell/\partial \rvx^0$. 

The effects of Jacobian on the signal propagation is studied in dynamical isometry and mean-field theory literature~\citenew{pennington2017resurrecting,saxe2013exact} and it is known that a network is said to satisfy dynamical isometry if the singular values of $\rmJ$ are concentrated near $1$.
Under this condition, error signals $\psi$ backpropagate isometrically through the network, approximately preserving its norm and all angles between error vectors. Thus, as dynamical isometry improves the trainability of the floating point networks, a similar technique can be useful for gradient based attacks as well. 

In fact, almost all initialization techniques (\eg,~\cite{glorot2010understanding}) approximately ensures that the Jacobian $\rmJ$ is well-conditioned for better trainability and it is hypothesized that approximate isometry is preserved even at the end of the training.
But, for \acrshort{BNN}s, the weights are constrained to be $\{-1,1\}$ and hence the weight distribution at end of training is completely different from the random initialization. 
Furthermore, it is not clear that fully-quantized networks can achieve well-conditioned Jacobian, which guided some research activity in utilizing layerwise scalars (either predefined or learned) to improve \acrshort{BNN} training~\citenew{mcdonnell2018training,rastegari2016xnor}. 
We would like to point out that the focus of this paper is to improve gradient based attacks on already trained \acrshort{BNN}s. 
To this end learning a new scalar to improve signal propagation at each layer is not useful as it can alter the decision boundary of the network and thus cannot be used in practice on already trained model. 

\SKIP{
\subsubsection{Improving Signal Propagation.} Recent methods have shown to improve signal propagation of the network by learning orthogonal weights. Such a mechanism can also be applied for trained networks where single scalar can be learnt at each layer to solve the following objective:
\begin{equation}
\begin{split}
\label{eqn:scalar_sp}
\min_{\rvc^l}\|(\rvc^l \odot \rmW)^T (\rvc^l \odot \rmW^l) -  \rmI^l\|_F,
\end{split}
\end{equation}
where $\rvc^l$ and $\rmI^l$ are the learnable scalars and identity matrix at layer $l$ respectively and $\|\cdot\|_F$ is the Frobenius norm. The above objective finds scalar such that scaled weights are close to orthogonality. This technique has been found to be effective in improving the signal propagation of the networks. However, for gradient ascent attacks, learning a new scalar to improve signal propagation at each layer $l$, can alter the decision
boundary of the network and thus cannot really be used in practise on already trained model. 
}
\paragraph{Temperature Scaling for better Signal Propagation.} 

In this paper, we propose to use a single scalar per network to improve the signal propagation of the network using temperature scaling. 
In fact, one could replace softmax with a monotonic function such that the prediction is not altered, however, we will show in our experiments that a single scalar with softmax has enough flexibility to improve signal propagation and yields almost $100\%$ success rate with \acrshort{PGD} attacks.
Essentially, we can use a scalar, $\beta > 0$ without changing the decision boundary of the network by preserving the relative order of the logits. 
Precisely, we consider the following:
\vspace{-1ex}
\begin{equation}
\begin{split}
\label{eqn:scalar_sp}
\rvp(\beta) = \softmax(\bar{\rva}^K)\ , \qquad \bar{\rva}^K = \beta\, \rva^K \ .
\end{split}
\end{equation}
Here, we write the softmax output probabilities $\bfp$ as a function of $\beta$ to emphasize that they are softmax output of temperature scaled logits.
Now since in this context, the only variable is the temperature scale $\beta$, we denote the loss and the error signal as functions of only $\beta$.
With this simplified notation the gradient of the temperature scaled loss with respect to the inputs can be written as:
\begin{equation}
\begin{split}
\label{eqn:grad_input_ts}
\frac{\partial \ell(\beta)}{\partial \rvx^0} = \frac{\partial \ell(\beta)}{\partial \bar{\rva}^K}\frac{\partial \bar{\rva}^K}{\partial \rva^K}\frac{\partial \rva^K}{\partial \rvx^0} = \psi(\beta) \beta\, \rmJ\ .
\end{split}
\end{equation}
Note that $\beta$ affects the input-output Jacobian linearly while it nonlinearly affects the error signal $\psi$.
\SKIP{
has changed and the input-output Jacobian has been scaled by the scalar $\beta$. The new error signal's $(\psi')$ j\textsuperscript{th} element can be computed as:
\begin{equation}
\begin{split}
\label{eqn:psi_ts}
\psi'_j = \frac{\partial L'}{\partial \bar{x_j}^K} = -(\delta_{kj} - p'_j), \qquad \delta_{kj} = \I[i=j],
\end{split}
\end{equation}
where $\bar{x}'_j$ is the j\textsuperscript{th} element of $\bar{\rvx}'$ and $p'_j$ is the j\textsuperscript{th} element of $\rvp'$, and $k$ is the index of the ground truth class of the input $\rvx_0$. 

So, this single scalar $\beta$ essentially allows to change the scale of the input-output Jacobian $\rmJ$ and error signal $\psi$ and therefore improve the overall signal propagation of the network which aids in better in gradient ascent attacks. 
\kartik{I dont understand what we can say about homogenous activation functions here and how they affect this scalar as in equation 11 this scalar is outside J and it does not matter whether $\phi$ is sigmoid or relu}
We would now discuss our proposed approach to choose this scalar $\beta$ such that we can improve the existing gradient attacks (\acrshort{FGSM}, \acrshort{PGD}).
}
To this end, we hope to obtain a $\beta$ that ensures the error signal is useful (\ie, not all zero) as well as the Jacobian is well-conditioned to allow the error signal to propagate to the input.

We acknowledge that while one can find a $\beta>0$ to obtain softmax output ranging from a uniform distribution ($\beta=0$) to one-hot vectors ($\beta\to\infty$), $\beta$ only scales the Jacobian. Therefore, if the Jacobian $\rmJ$ has zero singular values, our approach has no effect in those dimensions. 
However, since most of the modern networks consist of ReLU nonlinearities (generally positive homogeneous functions), the effect of a single scalar would be equivalent (ignoring the biases) to having layerwise scalars such as in~\cite{mcdonnell2018training}.
Thus, we believe a single scalar is sufficient for our purpose.

\section{Improved Gradients for Adversarial Attacks}
\SKIP{
\begin{wrapfigure}{r}{5.5cm}
    \vspace{-5ex}
    \includegraphics[width=0.4\textwidth]{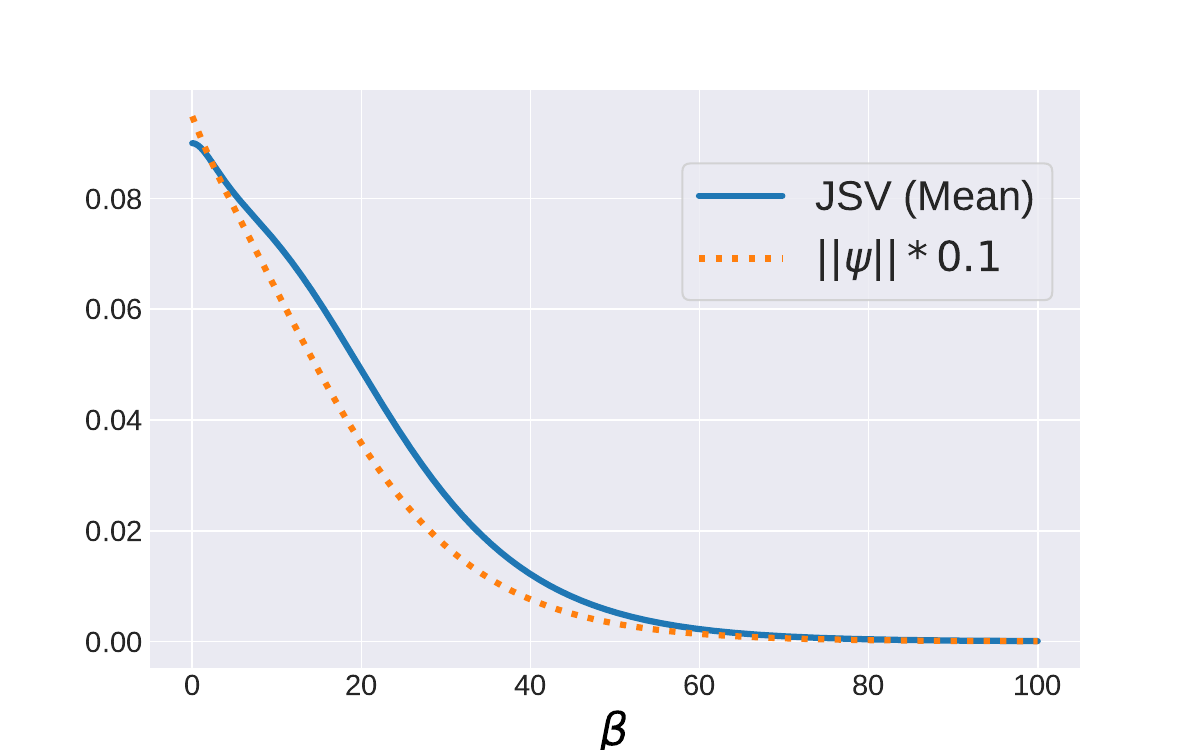}
    \vspace{-4ex}
    \caption{\em Error signal $(\psi(\beta))$ and Jacobian of softmax $(\partial \rvp(\beta)/\partial \bar{\rva}^K)$ \vs $\beta$ for a random correctly classified logits. 
    }
\label{fig:beta_vs_jacpsi}
    \vspace{-2.5ex}
\end{wrapfigure}
}

\SKIP{
\begin{figure}[!t]
    \centering
    \includegraphics[width=0.4\textwidth]{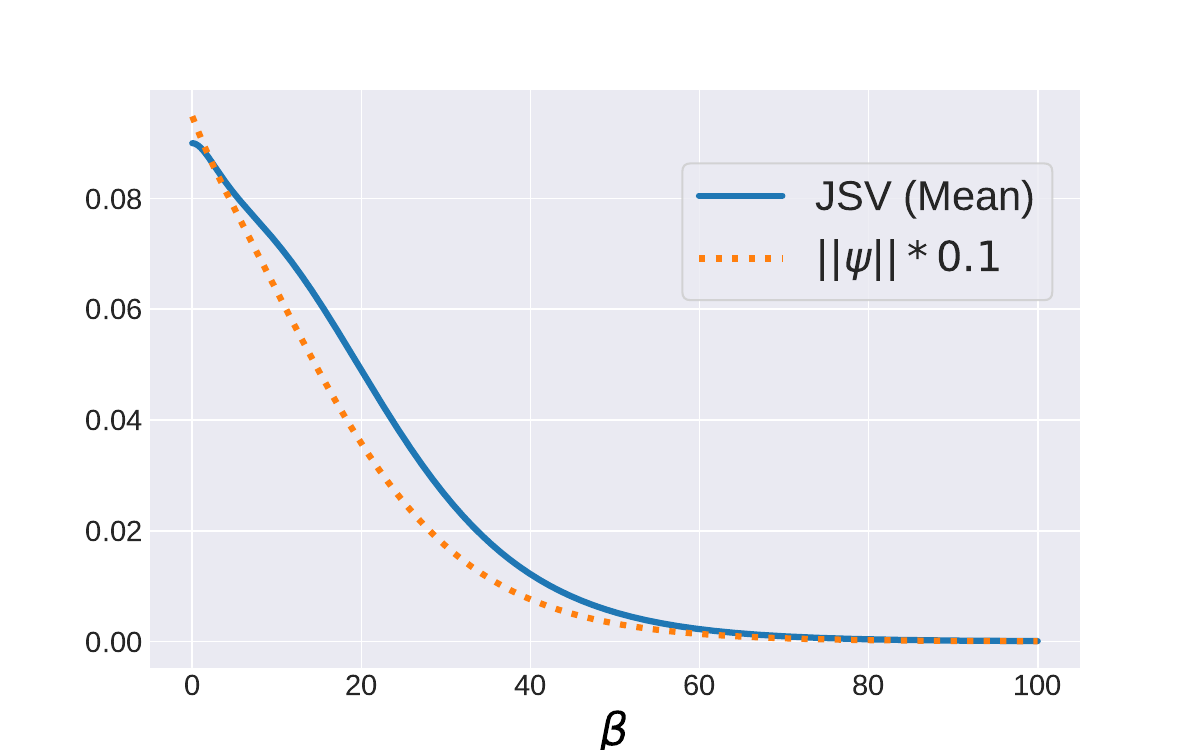}
    \caption{\em Error signal $(\psi(\beta))$ and Jacobian of softmax $(\partial \rvp(\beta)/\partial \bar{\rva}^K)$ \vs $\beta$ for a random correctly classified logits. 
    }
\label{fig:beta_vs_jacpsi}
\end{figure}

}
Now we discuss strategies to choose a scalar $\beta$ such that the gradients with respect to input are informative. 
Let us first analyze the effect of $\beta$ on the error signal.
To this end, 
\begin{equation}
\begin{split}
\label{eqn:psi_breakdown}
\psi(\beta) = \frac{\partial \ell(\beta)}{\partial \rvp(\beta)} \frac{\partial \rvp(\beta)}{\partial \bar{\rva}^K} = 
-(\bfy - \bfp(\beta))^T\ .
\end{split}
\end{equation}
where $\bfy$ is the one-hot encoded target label, and $\bfp(\beta)$ is the softmax output of scaled logits.

For adversarial attacks, we only consider the correctly classified images (\ie, $\argmax_j y_j = \argmax_j p_j(\beta)$) as there is no need to generate adversarial examples corresponding to misclassified samples.
From the above formula, it is clear that when $\bfp(\beta)$ is one-hot encoding then the error signal is $\bfzero$.
This is one of the reason for vanishing gradient issue in \acrshort{BNN}s. 
Even if this does not happen for a given image, one can increase $\beta\to \infty$ to make this error signal $\bfzero$.
Similarly, when $\bfp(\beta)$ is the uniform distribution, the norm of the error signal is at the maximum.
This can be obtained by setting $\beta=0$. 
However, this would also make $\partial \ell(\beta)/\partial \bfx^0 = \bfzero$ as the singular values of the input-output Jacobian would all be $0$. 

This analysis indicates that the optimal $\beta$ cannot be obtained by simply maximizing the norm of the error signal and we need to balance both the Jacobian as well as the error signal.
To summarize, the scalar $\beta$ should be chosen such that the following properties are satisfied:
\begin{tight_enumerate}
    \item $\|\psi(\beta)\|_2 > \rho$ for some $\rho>0$.
    \item The Jacobian $\beta\,\rmJ$ is well-conditioned, \ie, the singular values of $\beta\,\rmJ$ is concentrated around $1$.
\end{tight_enumerate}
\vspace{-2ex}
\subsubsection{\acrfull{NJS}.}

We now discuss a straightforward, two-step approach to attain the aforementioned properties.
Firstly, to ensure $\beta \rmJ$ is well-conditioned, we simply choose $\beta$ to be the inverse of the mean of singular values of $\rmJ$.
This guarantees that the mean of singular values of $\beta \rmJ$ is $1$. 

After this scaling, it is possible that the resulting error signal is very small. To ensure that $\|\psi(\beta)\|_2>\rho > 0$, we ensure that the softmax output $p_k(\beta)$ corresponding to the ground truth class $k$ is at least $\rho$ away from $1$.
We now state it as a proposition to derive $\beta$ given a lowerbound on $1-p_k(\beta)$.
\begin{pro}\label{thm:beta_cal_gt}
Let $\rva^K\in\R^d$ with $d >1$ and $a^K_1 \ge a^K_2 \ge \ldots \ge a^K_d$ and $a^K_1 - a^K_d = \gamma$. 
For a given ${0<\rho< (d-1)/d}$, there exists a $\beta > 0$ such that ${1- \softmax(\beta a^K_1)>\rho}$, then ${\beta < -\log(\rho/(d-1)(1-\rho))/\gamma}$. 
\end{pro}
\vspace{-2.5ex}
\begin{proof}
This is derived via a simple algebraic manipulation of softmax. Please refer to Appendix.
\end{proof}
This $\beta$ can be used together with the one computed using inverse of mean \acrfull{JSV}. We provide the pseudocode for our proposed \acrshort{PGD++} (\acrshort{NJS}) attack in Section A of Appendix. Similar approach can also be applied for \acrshort{FGSM++}. Notice that, this approach is simple and it adds negligible overhead to the standard \acrshort{PGD} attacks. However, it has a hand-designed hyperparameter $\rho$. To mitigate this, next we discuss a hyperparameter-free approach to obtain $\beta$. 

\subsubsection{\acrfull{HNS}.}
We now discuss another approach to obtain informative gradients.
Our idea is to maximize the Frobenius norm of the Hessian of the loss with respect to the input, where the intuition is that if the Hessian norm is large, then the gradient $\partial \ell/\partial \bfx^0$ is sensitive to an infinitesimal change in $\bfx^0$.
This means, the infinitesimal perturbation in the input is propagated in the forward pass to the last layer and propagated back to the input layer without attenuation (\ie, the returned signal is not zero), assuming there are no randomized or non-differentiable components in the network.
This clearly indicates that the network has good signal propagation as well as the error signals are not all zero. 
This objective can now be written as:
\begin{equation}
\begin{split}
\label{eqn:gradthresh}
\beta^* &= \argmax_{\beta>0} \norm{\frac{\partial^2 \ell(\beta)}{\partial (\rvx^0)^2}}_F \\
&= \argmax_{\beta>0} \norm{\beta \left[\psi(\beta) \frac{\partial \rmJ}{\partial \rvx^0} + \beta \left(\frac{\partial \rvp(\beta)}{\partial \bar{\rva}^K} \rmJ \right)^T \rmJ\right]}_F\ .\\[-1ex]
\end{split}
\end{equation}

\SKIP{
\begin{figure}
    \centering
    \includegraphics[width=0.4\textwidth]{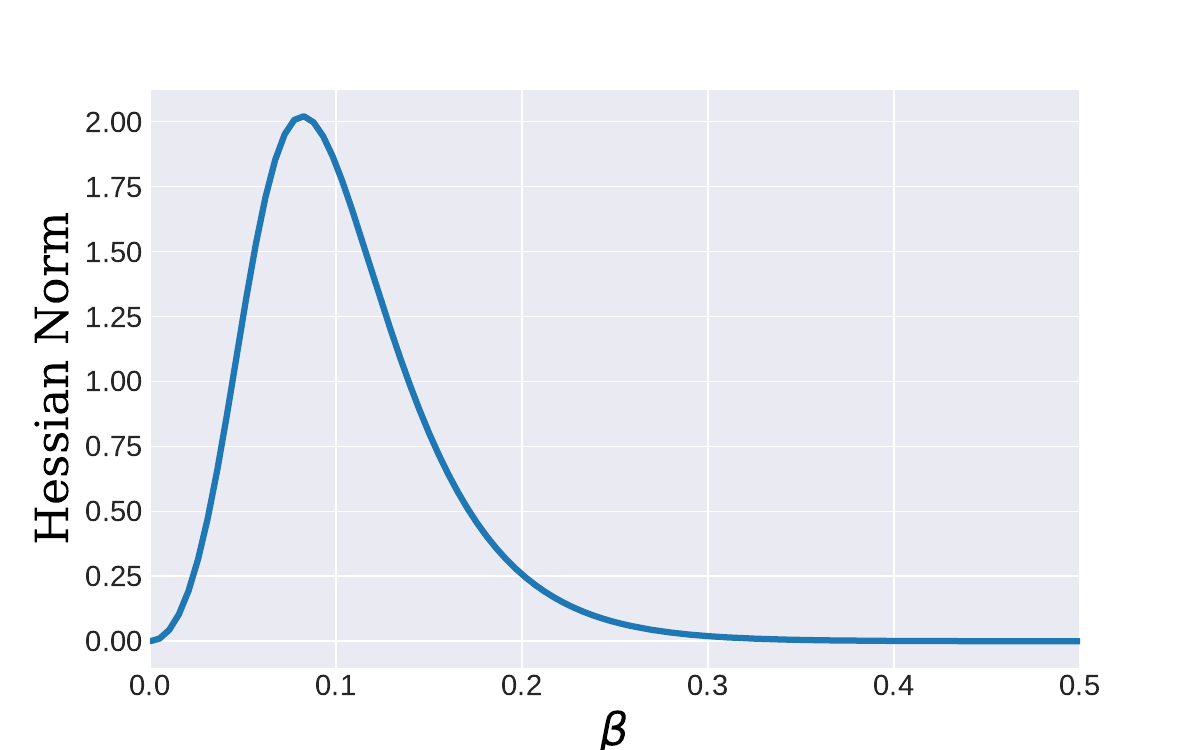}
    \caption{\em Hessian norm vs. $\beta$ on a random correctly classified image. The plot clearly shows a concave behaviour.
    }
\label{fig:beta_vs_HN}
\end{figure}
}
\SKIP{
\begin{wrapfigure}{r}{5.5cm}
    \vspace{-4ex}
    \includegraphics[width=0.4\textwidth]{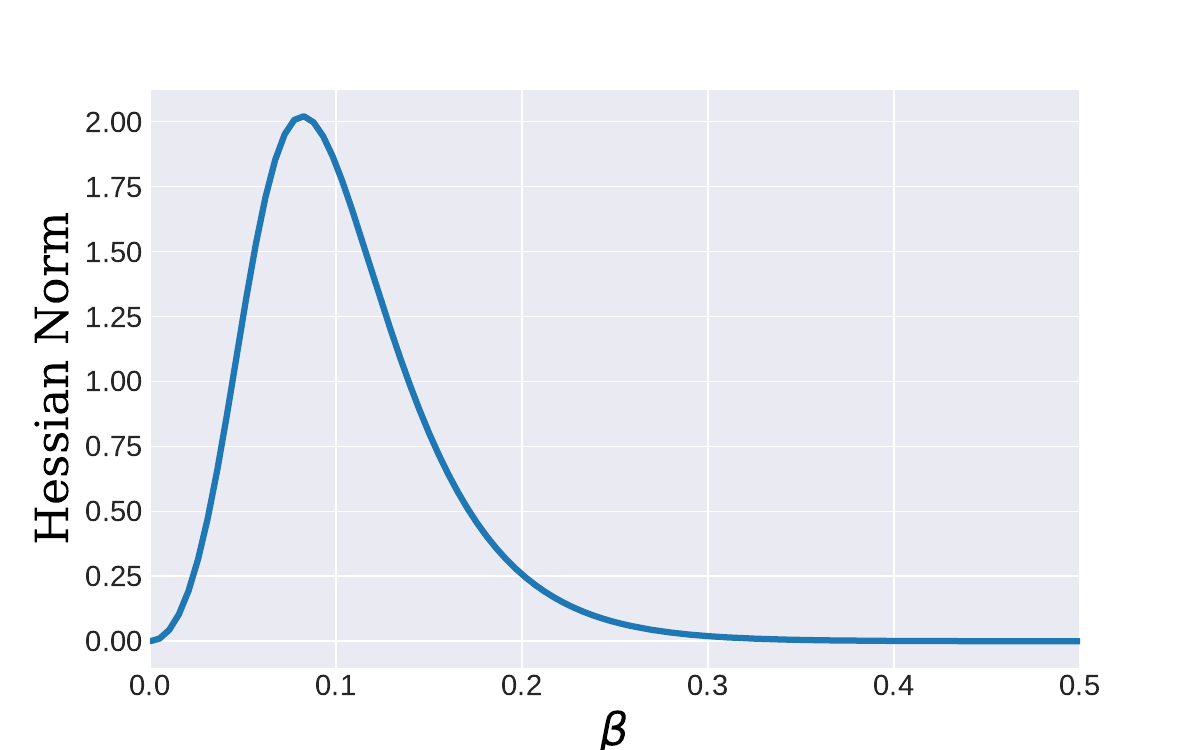}
    \vspace{-4.2ex}
    \caption{\em Hessian norm vs. $\beta$ on a random correctly classified image. The plot clearly shows a concave behaviour.
    }
\vspace{-2ex}
\label{fig:beta_vs_HN}
\end{wrapfigure}
}
The derivation is provided in Appendix.
Note, since $\rmJ$ does not depend on $\beta$, $\rmJ$ and $\partial \rmJ/\partial \rvx^0$ are computed only once, $\beta$ is optimized using grid search as it involves only a single scalar.
In fact, it is easy to see from the above equation that, when the Hessian is maximized, $\beta$ cannot be zero.
Similarly, $\psi(\beta)$ cannot be zero because if it is zero, then the prediction $\rvp(\beta)$ is one-hot encoding (\eqref{eqn:psi_breakdown}), consequently $\partial \bfp(\beta)/\partial \bar{\bfa}^K = \bfzero$ and this cannot be a maximum for the Hessian norm.
Hence, this ensures that $\|\psi(\beta^*)\|_2 >\rho$ for some $\rho >0$ and $\beta^*$ is bounded according to \proref{thm:beta_cal_gt}.
Therefore, the maximum is obtained for a finite value of $\beta$. 
Even though, it is not clear how exactly this approach would affect the singular values of the input-output Jacobian ($\beta\,\rmJ$), we know that they are finite and not zero.

Furthermore, there are some recent works~\citenew{moosavi2019robustness,qin2019adversarial} show that adversarial training makes the loss surface locally linear around the vicinity of training samples and enforcing local linearity constraint on loss curvature can achieve better robust to adversarial attacks. 
On the contrary, our idea of maximizing the Hessian, \ie, increasing the nonlinearity of $\ell$, could make the network more prone to adversarial attacks and we intend to exploit that. The psuedocode for \acrshort{PGD++} attack with \acrshort{HNS} is summarized in Section A of Appendix.


\section{Experiments}
We evaluate robustness accuracies of \acrshort{BNN}s with weight quantized (\acrshort{BNN-WQ}), weight and activation quantized (\acrshort{BNN-WAQ}), floating point networks (\acrshort{REF}).
We evaluate our two \acrshort{PGD}++ variants corresponding to \acrshort{HNS} and \acrshort{NJS} on \cifar{-10} and \cifar{-100} datasets with multiple network architectures. In order to demonstrate the effectiveness of our proposed variants on adversarially robust models, we also performed comparisons against  stronger attacks such as DeepFool~\citenew{moosavi2016deepfool} and \gls{BBA}~\citenew{brendel2019accurate} on adversarially trained \acrshort{REF} and \acrshort{BNN-WQ}. We further provide experimental comparisons against more recent gradient based/free attacks (Auto-PGD~\cite{croce2020reliable}, Square Attack~\cite{andriushchenko2020square}) proposed to alleviate the issue of gradient obfuscation. More analysis on signal propagation issue in \bnns{} and our variants success in improving it is provided in Section C.5 of Appendix.

\begin{table*}[!t]
\centering
\small
\begin{tabular}{cl|ccc|ccc|ccc}
\toprule
 & \multirow{3}{*}{\textbf{Network}} & \multicolumn{9}{c}{\textbf{Adversarial Accuracy (\%)}} \\
\cmidrule{3-11}
 &  & \multirow{2}{*}{\textbf{\acrshort{FGSM}}} & \multicolumn{2}{c|}{\textbf{\acrshort{FGSM++}}} & \multirow{2}{*}{\textbf{\acrshort{PGD}} ($L_\infty$)} & \multicolumn{2}{c|}{\textbf{\acrshort{PGD++}} ($L_\infty$)} & \multirow{2}{*}{\textbf{\acrshort{PGD}} ($L_2$)} & \multicolumn{2}{c}{\textbf{\acrshort{PGD++}} ($L_2$)} \\ 
 &  &  & \textbf{\acrshort{NJS}} & \textbf{\acrshort{HNS}} & & \textbf{\acrshort{NJS}} & \textbf{\acrshort{HNS}} & & \textbf{\acrshort{NJS}} & \textbf{\acrshort{HNS}} \\
 \midrule

\multirow{5}{*}{\rotatebox[origin=c]{90}{\textbf{CIFAR-10}}} & ResNet-18 & 40.49 & 3.46 & \textbf{2.51} & 26.98 & \textbf{0.00} & \textbf{0.00} & 55.68 & \textbf{0.05} & \textbf{0.05} \\  
& VGG-16 & 57.55 & 4.00 & \textbf{3.43} & 47.32 & \textbf{0.00} & \textbf{0.00} & 56.66 & \textbf{0.35} & 1.32 \\  
& ResNet-50 & 57.62 & 6.44 & \textbf{5.35} & 43.14 & \textbf{0.00} & \textbf{0.00} & 59.11 & 0.11 & \textbf{0.08} \\  
& DenseNet-121 & 26.80 & 4.67 & \textbf{4.24} & 9.11 & \textbf{0.00} & \textbf{0.00} & 45.78 & \textbf{0.03} & 0.06 \\  
& MobileNet-V2 & 33.50 & 6.42 & \textbf{5.42} & 26.86 & \textbf{0.00} & \textbf{0.00} & 34.40 & 0.12 & \textbf{0.09} \\ \midrule

\multirow{5}{*}{\rotatebox[origin=c]{90}{\textbf{CIFAR-100}}} & ResNet-18 & 25.22 & 14.08 & \textbf{1.80} & 8.23 & 2.45 & \textbf{0.00} & 25.20 & 6.79 & \textbf{0.26} \\ 
& VGG-16 & 19.82 & 7.98 & \textbf{1.76} & 17.44 & 0.88 & \textbf{0.16} & 16.25 & 3.17 & \textbf{0.63} \\ 
& ResNet-50 & 37.76 & 16.33 & \textbf{14.17} & 25.71 & \textbf{2.33} & 2.73 & 30.77 & 7.90 & \textbf{7.41} \\ 
& DenseNet-121 & 28.32 & 12.21 & \textbf{10.86} & 8.87 & 1.15 & \textbf{1.09} & 24.65 & 4.54 & \textbf{4.16} \\  
& MobileNet-V2 & 12.09 & 10.18 & \textbf{8.79} & 1.44 & \textbf{0.57} & 0.66 & 6.12 & 3.39 & \textbf{3.01} \\
\bottomrule
\end{tabular}
\caption{\em Adversarial accuracy on the test set for \acrshort{BNN-WQ}. Both our \acrshort{NJS} and \acrshort{HNS} variants consistently outperform original $L_\infty$ bounded \acrshort{FGSM} and \acrshort{PGD} attack, and $L_2$ bounded \acrshort{PGD} attack.}
\label{tab:bwn_mpgdlinf}
\end{table*}

\begin{table*}
\centering
\small
\begin{tabular}{cl|ccc|ccc|ccc}
\toprule
 & \multirow{3}{*}{\textbf{Network}} & \multicolumn{9}{c}{\textbf{Adversarial Accuracy (\%)}} \\
\cmidrule{3-11}
 &  & \multirow{2}{*}{\textbf{\acrshort{FGSM}}} & \multicolumn{2}{c|}{\textbf{\acrshort{FGSM++}}} & \multirow{2}{*}{\textbf{\acrshort{PGD}} ($L_\infty$)} & \multicolumn{2}{c|}{\textbf{\acrshort{PGD++}} ($L_\infty$)} & \multirow{2}{*}{\textbf{\acrshort{PGD}} ($L_2$)} & \multicolumn{2}{c}{\textbf{\acrshort{PGD++}} ($L_2$)} \\ 
 &  &  & \textbf{\acrshort{NJS}} & \textbf{\acrshort{HNS}} & & \textbf{\acrshort{NJS}} & \textbf{\acrshort{HNS}} & & \textbf{\acrshort{NJS}} & \textbf{\acrshort{HNS}} \\
 \midrule

\multirow{4}{*}{\rotatebox[origin=c]{90}{\textbf{\acrshort{REF}}}} & ResNet-18 & 7.62 & 5.55 & \textbf{5.35} & \textbf{0.00} & \textbf{0.00} & \textbf{0.00} & 1.12 & 0.09 & \textbf{0.05} \\ 
 & VGG-16 & 11.01 & 10.04 & \textbf{9.66} & 0.04 & \textbf{0.00} & \textbf{0.00} & 2.23 & \textbf{0.78} & 1.10 \\ 
 & ResNet-50 & 21.64 & 6.08 & \textbf{5.70} & 0.69 & \textbf{0.00} & \textbf{0.00} & 0.37 & \textbf{0.07} & 0.09 \\ 
 & DenseNet-121 & 11.40 & 7.58 & \textbf{7.30} & \textbf{0.00} & \textbf{0.00} & \textbf{0.00} & 0.65 & 0.08 & \textbf{0.06} \\ \midrule
\multirow{4}{*}{\rotatebox[origin=c]{90}{\textbf{\acrshort{BNN-WAQ}}}} & ResNet-18 & 40.84 & 19.46 & \textbf{19.09} & 8.57 & \textbf{0.03} & 0.04 & 25.24 & \textbf{2.33} & 2.59 \\  
 & VGG-16 & 79.92 & 15.96 & \textbf{15.39} & 78.01 & \textbf{0.01} & 0.02 & 85.62 & \textbf{0.49} & 0.62 \\  
 & ResNet-50 & 33.16 & \textbf{25.89} & 27.05 & 0.49 & \textbf{0.23} & 0.45 & 19.41 & \textbf{6.68} & 8.77 \\  
 & DenseNet-121 & 37.20 & \textbf{23.89} & 24.69 & 0.81 & \textbf{0.10} & 0.18 & 48.37 & \textbf{3.72} & 6.17 \\ 

\bottomrule
\end{tabular}
\caption{\em Adversarial accuracy on the test set of \cifar{-10} for \acrshort{REF} and \acrshort{BNN-WAQ}. 
Both our \acrshort{NJS} and \acrshort{HNS} variants consistently outperform original \acrshort{FGSM} and \acrshort{PGD} ($L_\infty/L_2$ bounded) attacks.}
\label{tab:ref_bnnwaq_mpgdlinf}
\end{table*}

\begin{table*}[t]
\centering
\small
\begin{tabular}{l|cccc|cccccc}
\toprule
 \multirow{3}{*}{\textbf{Network}} & \multicolumn{10}{c}{\textbf{Adversarial Accuracy (\%)}} \\
\cmidrule{2-11}
  & \multirow{2}{*}{\textbf{\acrshort{FGSM}}} & \textbf{\acrshort{FGSM}} & \multicolumn{2}{c|}{ \textbf{\acrshort{FGSM++}}} & \multirow{2}{*}{\textbf{\acrshort{PGD}}} & \textbf{\acrshort{PGD}} & \textbf{Deep} & \textbf{\acrshort{BBA}} & \multicolumn{2}{c}{ \textbf{\acrshort{PGD++}}} \\
 & & $\beta=0.1$ & \textbf{\acrshort{NJS}} & \textbf{\acrshort{HNS}} & & $\beta=0.1$ & \textbf{Fool} &  & \textbf{\acrshort{NJS}} & \textbf{\acrshort{HNS}} \\
 
 \midrule

\acrshort{REF} & 62.38 & 69.52 & 61.43 & \textbf{61.40} & 48.73 & 61.27 & 51.01 & 48.43 & \textbf{47.17} & 48.54 \\
\acrshort{BC} & 53.91 & 62.46 & 52.90 & \textbf{52.27} & 41.29 & 54.24 & 42.65 & 40.14 & 39.35 & \textbf{39.34} \\
\acrshort{GD}-$\tanh$ & 56.13 & 65.06 & 55.54 & \textbf{54.81} & 42.77 & 56.78 & 44.78 & 42.94 & \textbf{42.14} & 42.30 \\
\acrshort{MD}-$\tanh$-\acrshort{S} & 55.10 & 63.42 & 54.74 & \textbf{53.82} & 41.34 & 54.22 & 43.46 & 40.69 & 40.76 & \textbf{40.67} \\
\bottomrule
\end{tabular}
\caption{\em Adversarial accuracy on \cifar{-10} with \sresnet{-18} for adversarially trained \acrshort{REF} and \acrshort{BNN-WQ} using different quantization methods (\acrshort{BC}, \acrshort{GD}-$\tanh$, \acrshort{MD}-$\tanh$-\acrshort{S}). Our improved attacks are compared against \acrshort{FGSM}, \acrshort{PGD} ($L_\infty$), a heuristic choice of $\beta=0.1$, DeepFool and \acrshort{BBA}.
Albeit on adversarially trained networks, our methods outperform all the comparable methods.
}
\label{table:advtrain_++}
\end{table*}

 We use state of the art models trained for binary quantization (where all layers are quantized) for our experimental evaluations. We provide adversarial attack parameters used for \acrshort{FGSM}/\acrshort{PGD} in \tabref{tab:attack_params} of Appendix and for other attacks, we use default parameters used in Foolbox~\citenew{rauber2017foolbox}. We also provide some other experimental comparisons such as comparisons against combinatorial attack proposed in~\cite{khalil2018combinatorial} in the Appendix. For our \acrshort{HNS} variant, we sweep $\beta$ from a range such that the hessian norm is maximized for each image, as explained in Appendix. For our \acrshort{NJS} variant, we set the value of $\rho=0.01$. In fact, our attacks are not very sensitive to $\rho$ and we provide the ablation study in the Appendix. 

\subsection{Results}

Our comparisons against the original \acrshort{PGD} ($L_2$/$L_\infty$) and \acrshort{FGSM} attack for different \acrshort{BNN-WQ} are reported in~\tabref{tab:bwn_mpgdlinf}.
Our \acrshort{PGD++} variants consistently outperform original \acrshort{PGD} on all networks on both datasets. Even being a gradient based attack, our proposed \acrshort{PGD++} ($L_2$/$L_\infty$) variants can in fact reach adversarial accuracy close to $0$ on \cifar{-10} dataset, demystifying the fake robustness binarized networks tend to exhibit due to poor signal propagation. 

\begin{table*}
\centering
\small

\begin{tabular}{l|cccc|cccccc}
\toprule
 \multirow{3}{*}{\textbf{Network}} & \multicolumn{10}{c}{\textbf{Adversarial Accuracy (\%)}} \\
\cmidrule{2-11}
 & \multirow{2}{*}{\textbf{\acrshort{FGSM}}} & \textbf{\acrshort{FGSM}} & \multicolumn{2}{c|}{\textbf{\acrshort{FGSM++}}} & \multirow{2}{*}{\textbf{\acrshort{PGD}}} & \textbf{\acrshort{PGD}} & \multirow{2}{*}{\textbf{\acrshort{APGD}}} &
 \textbf{Square} &
 \multicolumn{2}{c}{\textbf{\acrshort{PGD++}}} \\ 
 &  & \textbf{(\acrshort{DLR})} & \textbf{\acrshort{NJS}} & \textbf{\acrshort{HNS}} &  & \textbf{(\acrshort{DLR})} & & \textbf{Attack} & \textbf{\acrshort{NJS}} & \textbf{\acrshort{HNS}} \\
 \midrule
 \acrshort{REF} & 7.62 & 19.48 & 5.55 & \textbf{5.35} & \textbf{0.00} & \textbf{0.00} & \textbf{0.00} & 0.55 & \textbf{0.00} & \textbf{0.00} \\  
 \acrshort{BNN-WQ} & 40.49 & 19.72 & 3.46 & \textbf{2.51} & 26.98 & \textbf{0.00} & \textbf{0.00} & 0.41 & \textbf{0.00} & \textbf{0.00} \\  
 \acrshort{BNN-WAQ} & 40.84 & 41.78 & 19.46 & \textbf{19.09} & 8.57 & 4.57 & 6.32 & 21.45 & \textbf{0.03} & 0.04 \\ 
 \midrule
 \acrshort{REF}$^*$ & 62.38 & 66.39 & 61.43 & \textbf{61.40} & 48.73 & 49.73 & 49.00 & 54.05 & \textbf{47.17} & 48.54\\  
 \acrshort{BNN-WQ}$^*$ & 55.10 & 59.14 & 54.74 & \textbf{53.82} & 41.34 & 41.42 & 40.85 & 46.67 & 40.76 & \textbf{40.67} \\  
\bottomrule
\end{tabular}
\caption{\em Adversarial accuracy for \acrshort{REF}, \acrshort{BNN-WQ}, and \acrshort{BNN-WAQ} trained on \cifar{-10} using \sresnet{-18}. Here $^*$ denotes adversarially trained models. Both our \acrshort{NJS} and \acrshort{HNS} variants consistently outperform $L_\infty$ bounded \acrshort{FGSM}, \acrshort{PGD} and \acrfull{APGD}~\citenew{croce2020reliable} attack performed using \acrfull{DLR} loss and a gradient free attack namely, Square Attack~\citenew{andriushchenko2020square} under $L_\infty$ bound (8/255). Notice, \acrshort{FGSM}, \acrshort{PGD} and \acrshort{APGD} attack with \acrshort{DLR} loss and Square Attack perform even worse than their original form on adversarially trained models in most cases.
}
\label{tab:compare_dlr}
\end{table*}

Similarly, for one step \acrshort{FGSM} attack, our modified versions outperform original \acrshort{FGSM} attacks by a significant margin consistently for both datasets on various network architectures. We would like to point out such an improvement in the above two attacks is considerably interesting, knowing the fact that \acrshort{FGSM}, \acrshort{PGD} with $L_\infty$ attacks only use the sign of the gradients so improved performance indicates, our temperature scaling indeed makes some zero elements in the gradient nonzero. We would like to point out here that one can use several random restarts to increase the success rate of original form of \acrshort{FGSM}/\acrshort{PGD} attack further but to keep comparisons fair we use single random restart for both original and modified attacks. Nevertheless, as it has been observed in \tabref{tab:adv_rob_bnn_cifar} even with 20 random restarts \acrshort{PGD} adversarial accuracies for \bnns{} cannot reach zero, whereas our proposed \acrshort{PGD++} variants consistently achieve perfect success rate.

The adversarial accuracies of \acrshort{REF} and \acrshort{BNN-WAQ} trained on \cifar{-10} using \sresnet{-18/50}, \svgg{-16} and \sdensenet{-121} for our variants against original counterparts are reported in \tabref{tab:ref_bnnwaq_mpgdlinf}. Overall, for both \acrshort{REF} and \acrshort{BNN-WAQ}, our variants outperform the original counterparts consistently. Particularly interesting, \acrshort{PGD++} variants improve the attack success rate on \acrshort{REF} networks. This effectively expands the applicability of our \acrshort{PGD++} variants and encourages to consider signal propagation of any trained network to improve gradient based attacks.
\acrshort{PGD++} with $L_\infty$ variants achieve near-perfect success rate on all \acrshort{BNN-WAQ}s, again validating the hypotheses of fake robustness of \acrshort{BNN}s.

\vspace{-0.4ex}
To further demonstrate the efficacy of proposed attack variants, we first adversarially trained the \acrshort{BNN-WQ}s (quantized using \acrshort{BC}~\citenew{courbariaux2015binaryconnect}, \acrshort{GD}-$\tanh$/\acrshort{MD}-$\tanh$-\acrshort{S}~\citenew{ajanthan2019mirror}) and floating point networks in a similar manner as in~\cite{madry2017towards}, using $L_\infty$ bounded \acrshort{PGD} with $T=7$ iterations, $\eta=2$ and $\epsilon=8$. 
We report the adversarial accuracies of $L_\infty$ bounded attacks and our variants on \cifar{-10} using \sresnet{-18} in~\tabref{table:advtrain_++}. These results further strengthens the usefulness of our proposed \acrshort{PGD++} variants. 
Moreover, with a heuristic choice of $\beta=0.1$ to scale down the logits before performing gradient based attacks performs even worse. 
Finally, even against stronger attacks (DeepFool, \acrshort{BBA}) under the same $L_\infty$ perturbation bound, our variants outperform consistently on these adversarially trained models in~\tabref{table:advtrain_++}. 
We would like to point out that our variants have negligible computational overhead over the original gradient based attacks, whereas stronger attacks are much slower in practice requiring 100-1000 iterations with an adversarial starting point (instead of random initial perturbation).

To illustrate the effectiveness of our proposed variants in improving signal propagation, we compare against gradient based attacks performed using recently proposed \acrfull{DLR} loss~\citenew{croce2020reliable} that aims to avoid the issue of saturating error signals. Also, we provide comparisons against recently introduced \acrfull{APGD} attack performed using \acrshort{DLR} loss and a gradient free attack namely, Square Attack~\citenew{andriushchenko2020square}. We show these experimental comparisons performed on \sresnet{-18} models trained on \cifar{-10} dataset in \tabref{tab:compare_dlr}. The attack parameters are same as used for the other experiments. It can be observed that our proposed variants perform better than both \acrshort{PGD} or \acrshort{APGD} with \acrshort{DLR} loss and Square Attack, consistently achieving 0\% adversarial accuracy. Infact, much computationally expensive Square attack is unable to achieve 0\% adversarial accuracy in any of the cases under the enforced $L_\infty$ bound.
The margin of difference is significant in case of \acrshort{FGSM} attack and adversarial trained models. Infact, it is important to note that gradient based attacks with \acrshort{DLR} loss and Square Attack perform worse on adversarially trained models than the original gradient based attacks.

\vspace{-0.33ex}
\paragraph{ImageNet.} For other large scale datasets such as ImageNet, BNNs are hard to train with full binarization of parameters and result in poor performance. Thus, most existing works~\citenew{yang2019quantization} on \bnns~keep the first and the last layers floating point and introduce several layerwise scalars to achieve good results on ImageNet. In such experimental setups, according to our experiments, trained BNNs do not exhibit gradient masking issues or poor signal propagation and thus are easier to attack using original \acrshort{FGSM}/\acrshort{PGD} attacks with complete success rate. In such experiments, our modified versions perform equally well compared to the original forms of these attacks.

\section{Related Work} 

Adversarial examples are first observed in~\cite{szegedy2013intriguing} and subsequently efficient gradient based attacks such as \gls{FGSM}~\citenew{goodfellow2014explaining} and \gls{PGD}~\citenew{madry2017towards} are introduced.
There exist recent stronger attacks such as~\cite{moosavi2016deepfool,carlini2016towards,yao2019trust,finlay2019logbarrier,brendel2019accurate}, however, compared to \gls{PGD}, they are much slower to be used for adversarial training in practice. For a comprehensive survey related to adversarial attacks, we refer the reader to~\cite{chakraborty2018adversarial}.

Some recent works focus on the adversarial robustness of \bnns{}~\citenew{bernhard2019impact,sen2019empir,galloway2018attacking,khalil2018combinatorial,lin2018defensive}, however, a strong consensus on the robustness properties of quantized networks is lacking.
In particular, while~\cite{galloway2018attacking} claims parameter quantized networks are robust to gradient based attacks based on empirical evidence,~\citenew{lin2018defensive} shows activation quantized networks are vulnerable to such attacks and proposes a defense strategy assuming the parameters are floating-point. Differently,~\cite{khalil2018combinatorial} proposes a combinatorial attack hinting that activation quantized networks would have obfuscated gradients issue. Though as shown in the paper, the combinatorial attack is not scalable and thus experiments were shown on only few layered MLPs trained on MNIST. \cite{sen2019empir} shows ensemble of mixed precision networks to be more robust than original floating point networks; however~\cite{tramer2020adaptive} later shows the presented defense method can be attacked with minor modification in the loss function.
In short, although it has been hinted that there maybe gradient masking in \bnns{} (especially in activation quantized networks), a thorough understanding is lacking on whether \bnns{} are robust, if not what is the reason for the failure of most commonly used gradient based attacks on binary networks. We answer this question in this paper and introduce improved gradient based attacks.

\section{Conclusion}
In this work, we have shown that both \acrshort{BNN-WQ} and \acrshort{BNN-WAQ} tend to show a fake sense of robustness on gradient based attacks due to poor signal propagation. To tackle this issue, we introduced our two variants of \acrshort{PGD++} attack, namely \acrshort{NJS} and \acrshort{HNS}. Our proposed \acrshort{PGD++} variants not only possess near-complete success rate on binarized networks but also outperform standard $L_\infty$ and $L_2$ bounded \acrshort{PGD} attacks on floating point networks. We finally show improvement in attack success on adversarially trained \acrshort{REF} and \acrshort{BNN-WQ} against stronger attacks (DeepFool and \acrshort{BBA}). 
\SKIP{Since \acrshort{PGD} has become a standard attack for adversarial training, our stronger \acrshort{PGD++} variants could provide a future scope of extending proposed \acrshort{PGD++} for improved robustness.
}

\appendix
\onecolumn
    {\centering
     {\huge \bf Appendix \\[2.5mm]  \par
    }
    }
    \bigskip

Here, we first provide the pseudocodes, proof of the proposition and the derivation of Hessian. Later we give additional experiments, analysis and the details of our experimental setting.
\section{Pseudocode}
We provide pseudocode for \acrshort{PGD++} with \acrshort{NJS} in \algref{alg:mpgd_mjsv} and \acrshort{PGD++} with \acrshort{HNS} in \algref{alg:mpgd_hess}.
\begin{algorithm}[H]
\begin{algorithmic}[1]
\caption{\acrshort{PGD++} with \acrshort{NJS} with $L_\infty$, $T$ iterations, radius $\epsilon$, step size $\eta$, network $f_{\rvw^*}$, input $\rvx^0$, label $k$, one-hot $\rvy\in\{0,1\}^d$, gradient threshold $\rho$.}
\label{alg:mpgd_mjsv}
\Require $T, \epsilon, \eta, \rho, \rvx^0, \rvy, k$ 
\Ensure $\| \rvx^{T+1} - \rvx^0 \|_\infty \leq \epsilon$
\State $\beta_1 = (M\, d)/\big(\sum_{i=1}^M \sum_{j=1}^d \mu_j(\rmJ_i)\big)$
\Comment $\beta_1$ computed using Network Jacobian.
\State $\rvx^1 = P_\infty^\epsilon(\rvx^0 + \mbox{Uniform}(-1, 1))$
\Comment{Random Initialization with Projection}
\For {$t \gets 1 ,\dots T$}
\State{$\beta_2 = 1.0$}
\State $\rvp' = \softmax(\beta_1 (f_{\rvw^*}(\rvx^t)))$
 \If{$1-p'_k \leq \rho$} 
 \Comment{$\rho=0.01$}
    \State {$\beta_2 = -\log(\rho/(d-1)(1-\rho))/\gamma$} 
    \Comment $\gamma$ computed using \proref{thm:beta_cal_gt_supp}
  \EndIf
\State $\ell =-\rvy^T \log(\softmax(\beta_2 \beta_1 (f_{\rvw^*}(\rvx^t))))$
\State $\rvx^{t+1} = P_\infty^\epsilon(\rvx^t + \eta \sign(\nabla_{\bfx}\ell(\bfx^t)))$
\Comment{Update Step with Projection}
\EndFor
\end{algorithmic}
\end{algorithm}

\begin{algorithm}[H]
\begin{algorithmic}[1]
\caption{\acrshort{PGD}++ with \acrshort{HNS} with $L_\infty$, $T$ iterations, radius $\epsilon$, step size $\eta$, network $f_{\rvw^*}$, input $\rvx^0$, label $k$, one-hot $\rvy\in\{0,1\}^d$, gradient threshold $\rho$.}
\label{alg:mpgd_hess}
\Require $T, \epsilon, \eta, \rvx^0, \rvy, k$ 
\Ensure $\| \rvx^{T+1} - \rvx^0 \|_\infty \leq \epsilon$
\State $\rvx^1 = P_\infty^\epsilon(\rvx^0 + \mbox{Uniform}(-1, 1))$
\Comment{Random Initialization with Projection}
\State $\beta^* = \argmax_{\beta>0} \norm{\partial^2 \ell(\beta)/\partial (\rvx^0)^2}_F$
\Comment Grid Search
\For {$t \gets 1 ,\dots T$}
\State $\ell =-\rvy^T \log(\softmax(\beta^* (f_{\rvw^*}(\rvx^t))))$
\State $\rvx^{t+1} = P_\infty^\epsilon(\rvx^t + \eta \sign(\nabla_{\bfx}\ell(\bfx^t)))$
\Comment{Update Step with Projection}
\EndFor
\end{algorithmic}
\end{algorithm}

\section{Derivations}

\subsection{B.1 ~Deriving $\beta$ given a lowerbound on $1-p_k(\beta)$}
\begin{pro}\label{thm:beta_cal_gt_supp}
Let $\rva^K\in\R^d$ with $d >1$ and $a^K_1 \ge a^K_2 \ge \ldots \ge a^K_d$ and $a^K_1 - a^K_d = \gamma$. 
For a given ${0<\rho< (d-1)/d}$, there exists a $\beta > 0$ such that ${1- \softmax(\beta a^K_1)>\rho}$, then ${\beta < -\log(\rho/(d-1)(1-\rho))/\gamma}$. 
\end{pro}
\begin{proof}
Assuming $a^K_1 - a^K_d = \gamma$, we derive a condition on $\beta$ such that ${1-\softmax(\beta a^K_1) > \rho}$.
\begin{align}
1- \softmax(\beta a^K_1)&>\rho\ ,\\\nonumber
\softmax(\beta a^K_1)&< 1- \rho\ ,\\\nonumber
\exp(\beta a^K_1)/\sum_{\lambda=1}^d\exp(\beta a^K_{\lambda}) &< 1- \rho\ ,\\\nonumber
1/\big(1 + \sum_{\lambda=2}^d\exp(\beta(a^K_{\lambda}-a^K_1))\big) &< 1- \rho\ .
\end{align}
Since, $a^K_1 - a^K_{\lambda}\leq\gamma$ for all $\lambda > 1$, 
\begin{equation}
    1/\big(1 + \sum_{\lambda=2}^d\exp(\beta(a^K_{\lambda}-a^K_1))\big) \leq 1/\big(1 + \sum_{\lambda=2}^d\exp(-\beta \gamma)\big)\ .
\end{equation}
Therefore, to ensure $1/\big(1 + \sum_{\lambda=2}^d\exp(\beta(a^K_{\lambda}-a^K_1))\big) < 1- \rho$, we consider,
\begin{align}
1/\big(1 + \sum_{\lambda=2}^d\exp(-\beta \gamma)\big) &< 1- \rho\ ,\quad\mbox{$a^K_1 - a^K_{\lambda}\leq\gamma$ for all $\lambda > 1$}\ ,\\\nonumber
1/\big(1 + (d-1)\exp(-\beta \gamma)\big) &< 1- \rho\ ,\\\nonumber
\exp(-\beta \gamma) &> \rho/(d-1)(1-\rho)\ ,\\\nonumber
-\beta\gamma &> \log(\rho/(d-1)(1-\rho))\ ,\quad\mbox{$\exp$ is monotone}\ ,\\\nonumber
\beta &< -\log(\rho/(d-1)(1-\rho))/\gamma\ .
\end{align}
Therefore for any ${\beta < -\log(\rho/(d-1)(1-\rho))/\gamma}$, the above inequality \\ $1- \softmax(\beta a^K_1)>\rho$ is satisfied.
\end{proof}

\subsection{B.2 ~Derivation of Hessian}
We now derive the Hessian of the input mentioned in Eq.~(\myref{8}) of the paper.
The input gradients can be written as:

\begin{align}
\frac{\partial \ell(\beta)}{\partial \rvx^0} = \frac{\partial \ell(\beta)}{\partial \rvp(\beta)} \frac{\partial \rvp(\beta)}{\partial \bar{\rva}^K(\beta)} \beta \rmJ = \psi(\beta) \beta \rmJ\ .
\end{align}
Now by product rule of differentiation, input hessian can be written as:
\begin{align}
\frac{\partial^2 \ell(\beta)}{\partial (\rvx^0)^2} &= \beta \left[\psi(\beta) \frac{\partial \rmJ}{\partial \rvx^0} +  \left(\frac{\partial \psi(\beta)}{\partial \rvx^0} \right)^T \rmJ\right]\ ,\\\nonumber
&= \beta \left[\psi(\beta) \frac{\partial \rmJ}{\partial \rvx^0} +  \left(\frac{\partial \rvp(\beta)}{\partial \rvx^0} \right)^T \rmJ\right]\ , \quad \psi(\beta)=-(\rvy - \rvp(\beta))^T\ ,\\\nonumber
&= \beta \left[\psi(\beta) \frac{\partial \rmJ}{\partial \rvx^0} +  \beta \left(\frac{\partial \rvp(\beta)}{\partial \bar{\rva}^K} \rmJ \right)^T \rmJ\right]\ .
\end{align}

\section{Additional Experiments}
In this section we first provide more experimental details and then some ablation studies.

\subsection{C.1 ~Experimental Details}
We first mention the hyperparameters used to perform \acrshort{FGSM} and \acrshort{PGD} attack for all the experiments in the paper in \tabref{tab:attack_params}. To make a fair comparison, we keep the attack parameters same for our proposed variants of \acrshort{FGSM++} and \acrshort{PGD++} attacks. In our experiments we found that increasing attack iterations, attack radius and step size \textit{does not yield better success rate for original gradient based attacks} thus we used these standard attack parameters. For attacks under $L_2$ bound, we use attack parameters similar to~\cite{finlay2019logbarrier}. All our experiments are performed using single NVIDIA Tesla P100 GPUs. We also perform multiple runs for achieving trained \bnns{} and the performance change remains $\leq 0.3 \%$ reflecting the significance of improvement in our attack success rates.

\SKIP{
\begin{wraptable}{r}{6.5cm}
\small
\vspace{-1ex}
\begin{tabular}{l|l|ccc}
\toprule
\textbf{Dataset}                    & \textbf{Attack}     & \textbf{$\epsilon$} & \textbf{$\eta$} & \textbf{$T$}  \\
\midrule
\multirow{3}{*}{\textbf{\cifar{-10}}}  & \textbf{\acrshort{FGSM}}      & 8       & 8   & 1  \\
                           & \textbf{\acrshort{PGD}} ($L_\infty$) & 8       & 2   & 20 \\
                           & \textbf{\acrshort{PGD}} ($L_2$)   & 120     & 15  & 20 \\
\midrule
\multirow{3}{*}{\textbf{\cifar{-100}}} & \textbf{\acrshort{FGSM}}       & 4       & 4   & 1  \\
                           & \textbf{\acrshort{PGD}} ($L_\infty$) & 4       & 1   & 10 \\
                           & \textbf{\acrshort{PGD}} ($L_2$)   & 60      & 15  & 10 \\
\bottomrule
\end{tabular}
\vspace{-1ex}
\caption{\em Attack parameters ($\epsilon$ \& $\eta$ in pixels).
}
\vspace{-5ex}
\label{tab:attack_params}
\end{wraptable}
}

\begin{table}[t]
\centering
\small
\begin{tabular}{l|l|ccc}
\toprule
\textbf{Dataset}                    & \textbf{Attack}     & \textbf{$\epsilon$} & \textbf{$\eta$} & \textbf{$T$}  \\
\midrule
\multirow{3}{*}{\textbf{\cifar{-10}}}  & \textbf{\acrshort{FGSM}}      & 8       & 8   & 1  \\
                           & \textbf{\acrshort{PGD}} ($L_\infty$) & 8       & 2   & 20 \\
                           & \textbf{\acrshort{PGD}} ($L_2$)   & 120     & 15  & 20 \\
\midrule
\multirow{3}{*}{\textbf{\cifar{-100}}} & \textbf{\acrshort{FGSM}}       & 4       & 4   & 1  \\
                           & \textbf{\acrshort{PGD}} ($L_\infty$) & 4       & 1   & 10 \\
                           & \textbf{\acrshort{PGD}} ($L_2$)   & 60      & 15  & 10 \\
\bottomrule
\end{tabular}
\caption{\em Attack parameters ($\epsilon$ \& $\eta$ in pixels).
}
\label{tab:attack_params}
\end{table}

For \acrshort{PGD++} with \acrshort{HNS} variant, we maximize Frobenius norm of Hessian with respect to the input as specified in Eq.~(\myref{8}) of the paper by grid search for the optimum $\beta$. We would like to point out that since only $\psi(\beta)$ and $\rvp(\beta)$ terms are dependent on $\beta$, we do not need to do forward and backward pass of the network multiple times during the grid search. This significantly reduces the computational overhead during the grid search. We can simply use the same network outputs $\rva^K$ and network jacobian $\rmJ$ (as computed without using $\beta$) for the grid search, while computing the other terms at each iteration of grid search. We apply grid search to find the optimum beta between $100$ equally spaced intervals of $\beta$ starting from $\beta_1$ to $\beta_2$. Here, $\beta_1$ and $\beta_2$ are computed based on Proposition \myref{1} in the paper where $\rho=1e-72$ and $\rho=1-(1/d)-(1e-2)$ respectively, where $d$ is number of classes and $\gamma=a_1^K-a_2^K$ so that $1-\softmax(\beta a_1^K) < \rho$. Also, note that we estimate the optimum $\beta$ for each test sample only at the start of the first iteration of an iterative attack and then use the same $\beta$ for the next iterations.
\paragraph{Computational Overhead of \acrshort{NJS} and \acrshort{HNS}.} Our Jacobian calculation takes just a single backward pass through the network and thus adds a negligible overhead. Our \acrshort{NJS} approach for scaling estimates $\beta$ as inverse of mean \acrshort{JSV} using $100$ random test samples, which is similar to $100$ backward passes. The computation overhead of \acrshort{HNS} is roughly comprised of a full backward pass through the network for estimating the Jacobian and grid search for finding optimal $\beta$.

\begin{wraptable}{r}{6.5cm}
    \small
    \centering
    \begin{tabular}{ccc}
        \toprule
          \textbf{\acrshort{FGSM}} & \textbf{\acrshort{FGSM++} (\acrshort{NJS})} & \textbf{\acrshort{FGSM++} (\acrshort{HNS})} \\
          \midrule
        0.01492 & 0.01604 & 0.0932 \\
        \bottomrule
    \end{tabular}
    \caption{\em Comparisons of computational cost in seconds between \acrshort{FGSM} and \acrshort{FGSM++} to generate one adversarial sample using different methods. Note, that both variants take less than 0.1 second which is significantly faster than stronger attacks mentioned in Table 4 and 5 of the main paper.}
    \label{tab:comp_cost}
\end{wraptable}

For \acrshort{HNS} variant, we maximize the Frobenius norm of the Hessian as specified in Eq.~(\myref{8}) of the paper by grid search to find the optimum $\beta$. First of all, we are computing the Hessian w.r.t. input (dimension in thousands), this is significantly cheaper than Hessian w.r.t. network parameters (dimension in millions). Furthermore, in Eq.~(\myref{8}), $\rmJ$ does not depend on $\beta$, therefore $\rmJ$ and $\partial \rmJ/\partial \rvx^0$ are \textit{computed only once during the grid search} and re-used for each iteration of grid search. This significantly reduces the computational overhead. Our Jacobian calculation takes just a single full backward pass through the network (using parallel compute), thus adds negligible overhead. Moreover, for piecewise linear networks (\eg., ReLU activations), $\partial \rmJ/\partial \rvx^0=0$ almost everywhere \citenew{yao2018hessian}. We provide computational time comparisons for original \acrshort{FGSM} and \acrshort{FGSM++} using \acrshort{NJS} and \acrshort{HNS} on \sresnet{-18} using \cifar{-10} in \tabref{tab:comp_cost}. These experiments are performed on NVIDIA GeForce RTX 2080 GPUs at a batchsize of $1$. Note, that our implementation of sequential grid search with $100$ intervals can be parallelized to speed up \acrshort{HNS} further.

\subsection{C.2 ~Comparisons against Auto-\acrshort{PGD} attack, gradient free attack and Combinatorial attack}

We also compared our proposed \acrshort{PGD++} variants against recently proposed \acrfull{APGD} with \acrfull{DLR} loss~\citenew{croce2020reliable} and gradient free Square Attack~\citenew{andriushchenko2020square} on different networks trained using \svgg{-16} on \cifar{-10} dataset and the results are reported in \tabref{tab:apgd_square}. The attack parameters for this experiment are the same as reported in the paper. It can be clearly seen that our proposed variants perform better than both \acrshort{APGD} with \acrshort{DLR} loss and Square Attack, consistently achieving 0\% adversarial accuracy. Infact, much computationally expensive Square attack is unable to achieve 0\% adversarial accuracy in any of the cases under the enforced $L_\infty$ bound. 

\SKIP{\begin{wrapfigure}{r}{6.5cm}
    \vspace{-1ex}
    \includegraphics[width=0.4\textwidth]{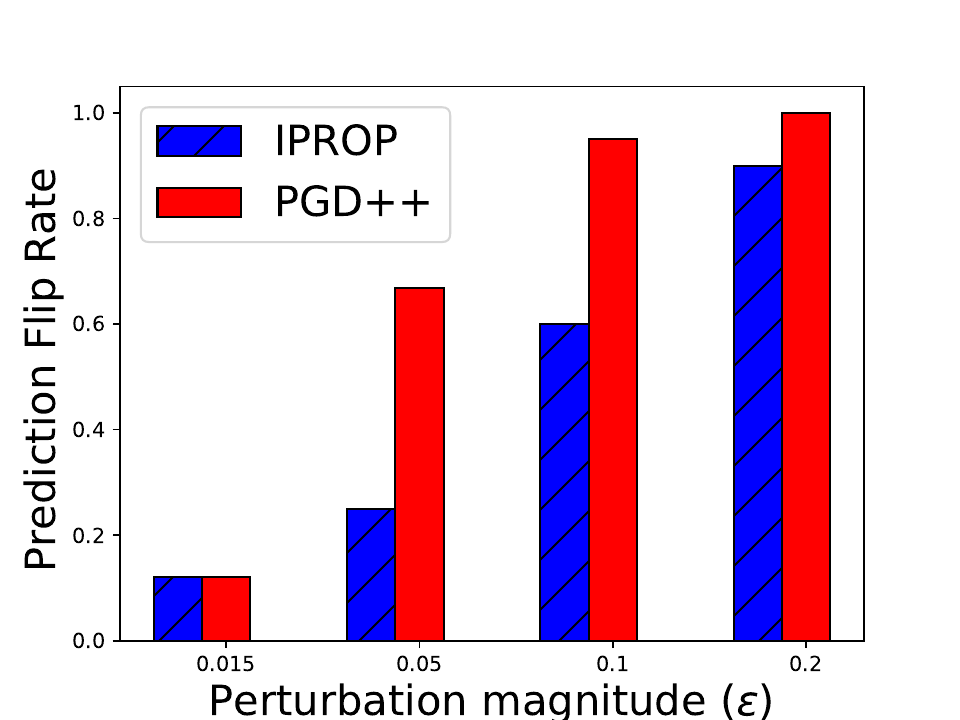}
    \caption{\em Classification flip rate (higher the better) using \acrshort{PGD++} (\acrshort{NJS}) attack and IPROP~\cite{khalil2018combinatorial} on MLP trained on MNIST dataset under various $L_\infty$ bound. Notice, \acrshort{PGD++} attack consistently flips the predicition more times than IPROP under different perturbation magnitudes.
    }
\vspace{-4ex}
\label{fig:iprop_comp}
\end{wrapfigure}}
\begin{figure}{}
    \centering
    \includegraphics[width=0.4\textwidth]{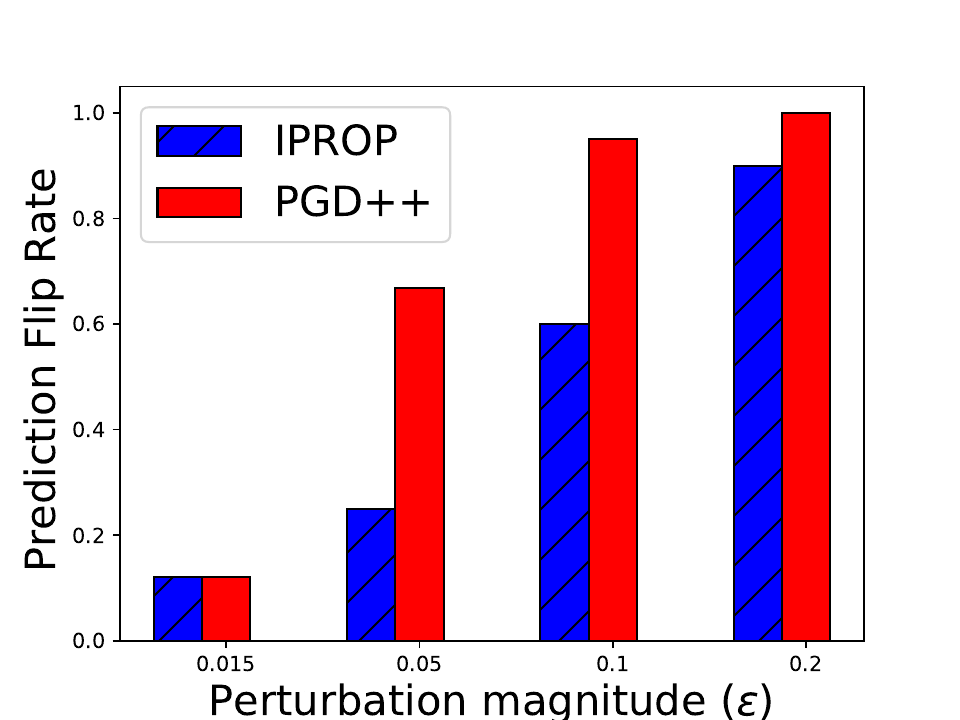}
    \caption{\em Classification flip rate (higher the better) using \acrshort{PGD++} (\acrshort{NJS}) attack and IPROP~\cite{khalil2018combinatorial} on MLP trained on MNIST dataset under various $L_\infty$ bound. Notice, \acrshort{PGD++} attack consistently flips the predicition more times than IPROP under different perturbation magnitudes.
    }
\label{fig:iprop_comp}
\end{figure}

We provide comparisons against combinatorial attack designed for quantized networks proposed in~\cite{khalil2018combinatorial}. Since their proposed attack namely \acrshort{IPROP} is not scalable for deep neural networks and was thus tested on MLPs trained on MNIST dataset. We follow their experimental protocol with a \acrshort{BNN-WAQ} using MLP (5 hidden layer with 100 units each) trained on MNIST dataset achieving 95\% classification accuracy. We provide the proportion of examples where \acrshort{IPROP} and \acrshort{PGD++} (\acrshort{NJS}) attack is able to flip the classification in \figref{fig:iprop_comp} under different $L_\infty$ perturbation bounds in range $[0,1]$. Our \acrshort{PGD++} variant performs consistently better than \acrshort{IPROP} under different perturbation magnitudes, clearly reflecting efficacy of \acrshort{PGD++} despite being simple modification over \acrshort{PGD}.

\begin{table}[t]
\centering
\small
\begin{tabular}{l|cccc}
\toprule
 \multirow{2}{*}{\textbf{Method}} & \multirow{2}{*}{\textbf{\acrshort{APGD}}} & \textbf{Square} & \multicolumn{2}{c}{\textbf{\acrshort{PGD++}}} \\ 
 &  & \textbf{Attack} & \textbf{\acrshort{NJS}} & \textbf{\acrshort{HNS}} \\
 \midrule
 \textbf{\acrshort{REF}} & 0.79 & 2.25 & \textbf{0.00} & \textbf{0.00} \\  
 \textbf{\acrshort{BNN-WQ}} & 8.23 & 1.98 & \textbf{0.00} & \textbf{0.00} \\  
 \textbf{\acrshort{BNN-WAQ}} & 0.38 & 16.67 & \textbf{0.01} & 0.02 \\  
\bottomrule
\end{tabular}
\vspace{1ex}
\caption{\em Adversarial accuracy for \acrshort{REF}, \acrshort{BNN-WQ} and \acrshort{BNN-WAQ} trained on \cifar{-10} using \svgg{-16}. Both our \acrshort{NJS} and \acrshort{HNS} variants consistently outperform \acrfull{APGD}~\citenew{croce2020reliable} performed using \acrfull{DLR} loss and a gradient free attack namely, Square Attack~\citenew{andriushchenko2020square} under $L_\infty$ bound (8/255).}
\label{tab:apgd_square}
\end{table}
\subsection{C.3 ~Other Experiments}

\begin{table*}[t]
\centering
\small
\begin{tabular}{l|ccc|ccc|ccc}
\toprule
 \multirow{3}{*}{\textbf{Network}} & \multicolumn{9}{c}{\textbf{Adversarial Accuracy (\%)}} \\
\cmidrule{2-10}
 & \multirow{2}{*}{\textbf{\acrshort{FGSM}}} & \multicolumn{2}{c}{\textbf{\acrshort{FGSM++}}} & \multirow{2}{*}{\textbf{\acrshort{PGD}} ($L_\infty$)} & \multicolumn{2}{c}{\textbf{\acrshort{PGD++}} ($L_\infty$)} & \multirow{2}{*}{\textbf{\acrshort{PGD}} ($L_2$)} & \multicolumn{2}{c}{\textbf{\acrshort{PGD++}} ($L_2$)} \\ 
 &  & \textbf{\acrshort{NJS}} & \textbf{\acrshort{HNS}} & & \textbf{\acrshort{NJS}} & \textbf{\acrshort{HNS}} & & \textbf{\acrshort{NJS}} & \textbf{\acrshort{HNS}} \\
 \midrule

 \textbf{ResNet-18} & 9.06 & 9.23 & \textbf{2.70} & 0.14 & 0.14 & \textbf{0.00} & 4.86 & 0.17 & \textbf{0.15} \\ 
 \textbf{VGG-16} & 16.28 & 17.24 & \textbf{9.19} & 1.53 & 0.95 & \textbf{0.25} & 4.87 & 1.50 & \textbf{1.38} \\ 
 \textbf{ResNet-50} & 12.95 & 12.95 & \textbf{11.94} & 0.12 & \textbf{0.00} & \textbf{0.00} & 10.50 & 4.43 & \textbf{4.14} \\ 
 \textbf{DenseNet-121} & 11.41 & 11.41 & \textbf{10.74} & \textbf{0.00} & \textbf{0.00} & \textbf{0.00} & 4.32 & 3.09 & \textbf{2.76} \\

\bottomrule
\end{tabular}
\vspace{1ex}
\caption{\em Adversarial accuracy on the test set of \cifar{-100} for \acrshort{REF} (floating point networks).
Both our \acrshort{NJS} and \acrshort{HNS} variants consistently outperform original \acrshort{FGSM} and \acrshort{PGD} ($L_\infty/L_2$ bounded) attacks.}
\label{tab:ref_mpgdlinf}
\end{table*}

\begin{table}[t]
\small
\centering
\begin{tabular}{l|ccccc}
\toprule
 \multirow{2}{*}{\textbf{Network}} & \multirow{2}{*}{\textbf{\acrshort{PGD}}} & \textbf{Deep} & \textbf{\acrshort{BBA}} & \multicolumn{2}{c}{ \textbf{\acrshort{PGD++}}} \\
 
 & & \textbf{Fool} &  & \textbf{\acrshort{NJS}} & \textbf{\acrshort{HNS}} \\
 \midrule
\textbf{\sresnet{-18}}  &  8.57 & 18.92 & 0.81 & \textbf{0.03}  & 0.04 \\
\textbf{\svgg{-16}} & 78.01 & 12.12 & 0.10  & \textbf{0.01} & 0.02 \\
\bottomrule
\end{tabular}
\vspace{1ex}
\caption{\em Adversarial accuracy on the test set of \cifar{-10} for \acrshort{BNN-WAQ}. Here, we compare our proposed variants against much stronger attacks namely DeepFool~\citenew{moosavi2016deepfool} and \acrshort{BBA}~\citenew{brendel2019accurate}. Both our variants outperform stronger attacks. Note, DeepFool and \acrshort{BBA} are much slower in practise requiring 100-1000 iterations. \acrshort{BBA} specifically requires even an adversarial start point that needs to be computed using another adversarial attack.
}
\label{tab:bnnwaq_strongerattacks}
\end{table}

We provide adversarial accuracy comparisons for different attack methods on \cifar{-100} using \sresnet{-18}, \svgg{-16}, \sresnet{-50} and \sdensenet{-121} in \tabref{tab:ref_mpgdlinf}. Again similar to the results in the paper, our proposed \acrshort{PGD++} and \acrshort{FGSM++} outperform original form of \acrshort{PGD} and \acrshort{FGSM} consistently in all the experiments on floating point networks. We also provide adversarial accuracy comparison of our proposed variants against stronger attacks namely DeepFool~\citenew{moosavi2016deepfool} and \acrshort{BBA}~\citenew{brendel2019accurate} on \acrshort{BNN-WAQ} trained on \cifar{-10} dataset in \tabref{tab:bnnwaq_strongerattacks}. In this experiment, our proposed variants again outperform even the stronger attacks which take 100-1000 iterations with adversarial start point (instead of random initial perturbation). It should be noted that although \acrshort{BBA} performs much better than DeepFool and \acrshort{PGD}, it still has inferior success rate than ours considering the fact that it takes multiple hours to run \acrshort{BBA} whereas our proposed variants are almost as efficient as \acrshort{PGD} attack.

\begin{wrapfigure}{r}{6.5cm}
    \vspace{-2.3ex}
    \includegraphics[width=0.4\textwidth]{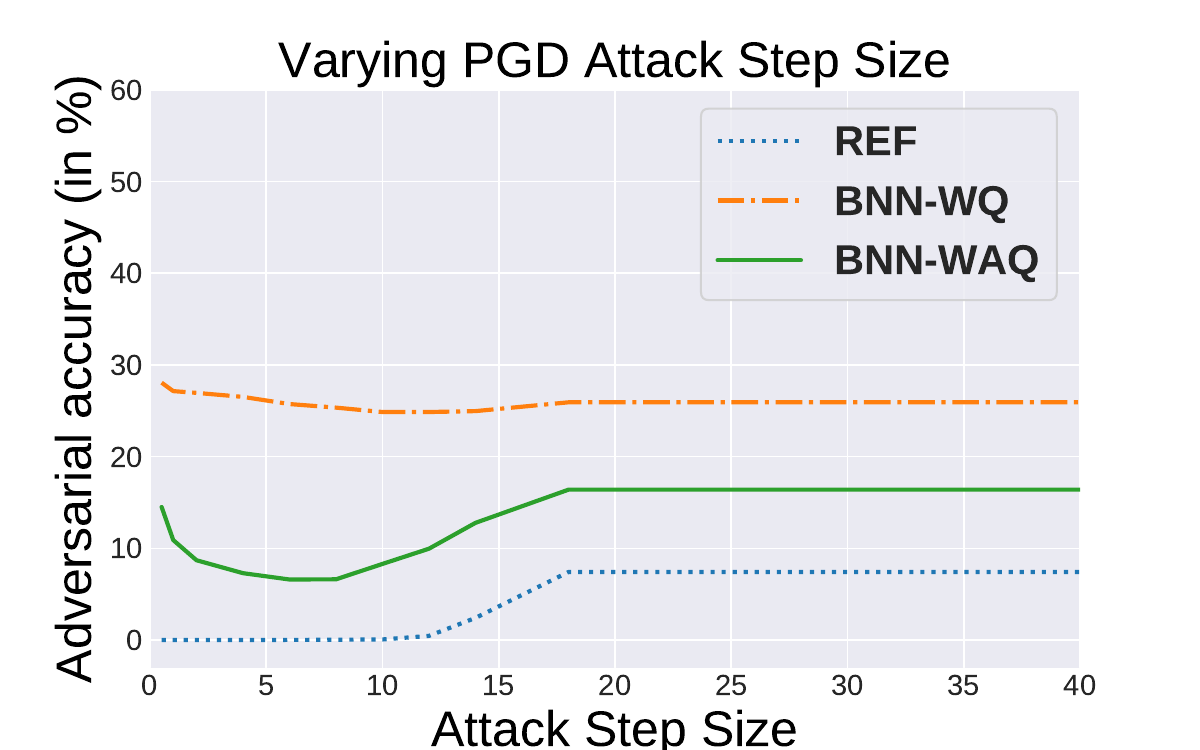}
    \caption{\em Adversarial accuracy using \acrshort{PGD} attack under $L_\infty$ bound (8/255) with varying step size ($\eta$) on \sresnet{-18} trained on \cifar{-10}. Notice, \acrshort{PGD} attack is unable to reach zero adversarial accuacy for \bnns{} with any step size.
    }
\vspace{-4ex}
\label{fig:adv_vs_stepsize}
\end{wrapfigure}
\paragraph{Step Size Tuning for \acrshort{PGD} attack.} We would like to point out that step size $\eta$ and temperature scale $\beta$ have different effects in the attacks performed. Notice, \acrshort{PGD} and \acrshort{FGSM} attack under $L_\infty$ bound only use the sign of input gradients in each gradient ascent step. Thus, if the input gradients are completely saturated (which is the case for \bnns{}), original forms of \acrshort{PGD} or \acrshort{FGSM} will not work irrespective of the step size used. To illustrate this, we performed extensive step size tuning for original form of \acrshort{PGD} attack on different \sresnet{-18} models trained on \cifar{-10} dataset and the adversarial accuracies are reported in \figref{fig:adv_vs_stepsize}. It can be observed clearly that although tuning the step size lowers adversarial accuracy a bit in some cases but still cannot reach zero for \bnns{} unlike our proposed variants.

\paragraph{Adversarial training using \acrshort{PGD++}.} We also investigate the potential application of \acrshort{PGD++} for adversarial training to improve the robustness of neural networks. \acrshort{PGD++} attack is most effective when applied to a network with poor signal propagation. However, adversarial training is performed from random initialization \citenew{glorot2010understanding} exhibiting good signal propagation. Thus, \acrshort{PGD} and \acrshort{PGD++} perform similarly for adversarial training. We infer these conclusions from our experiments on adversarial training using \acrshort{PGD++}.

\begin{figure*}[t]
        \begin{subfigure}{0.48\textwidth}
    \includegraphics[width=\textwidth]{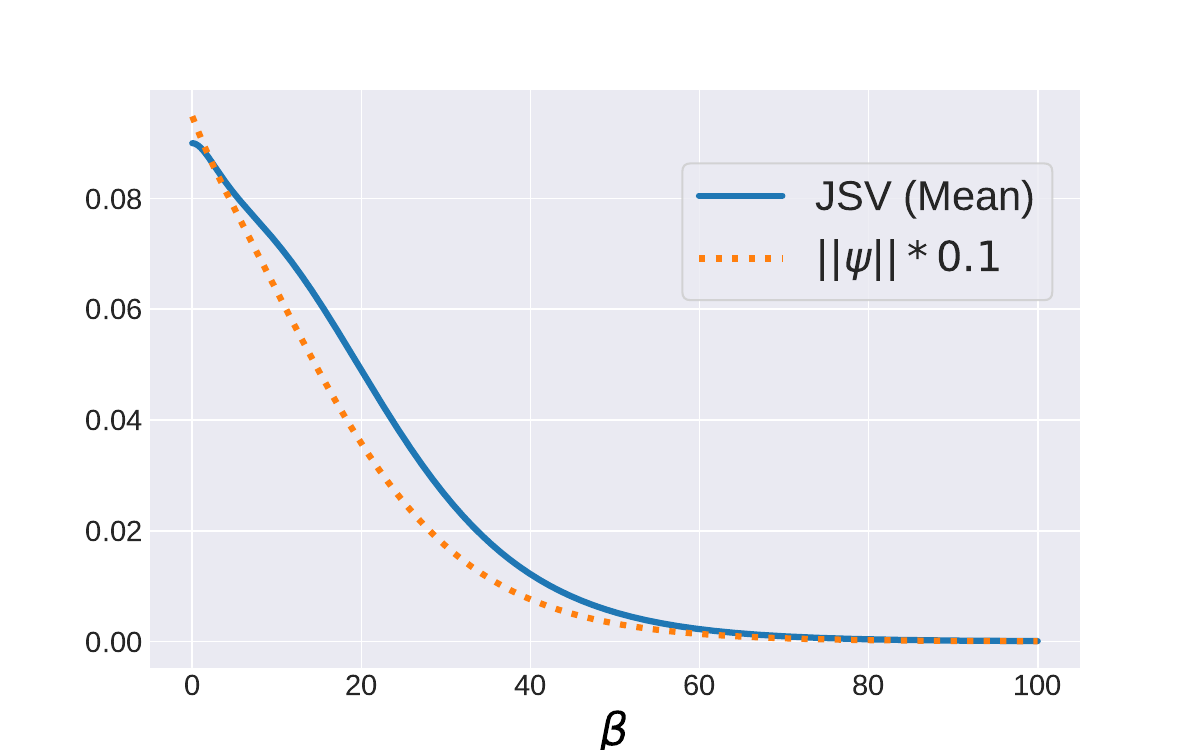}
        \caption{}          
        \end{subfigure} 
        \begin{subfigure}{0.48\textwidth}
    \includegraphics[width=\textwidth]{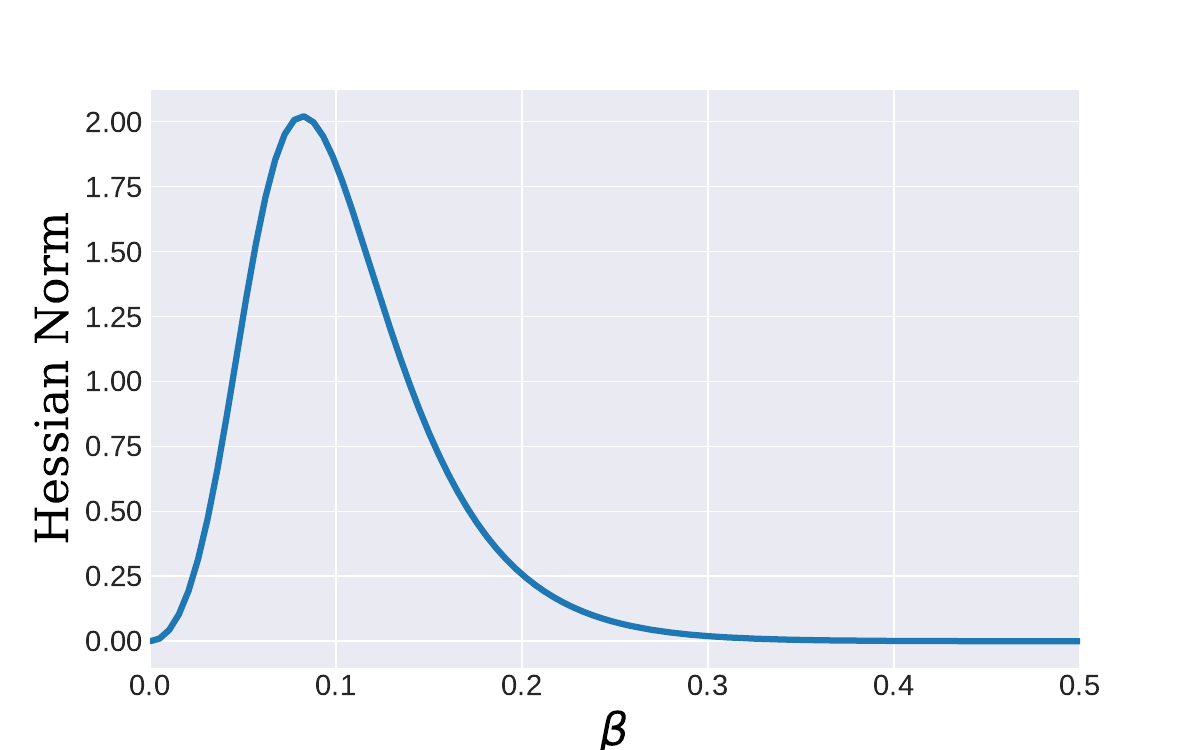}
    \caption{}        
        \end{subfigure}%
    \caption{\em Plots to show how variation in $\beta$ affects jacobian, error signal and hessian. (a) Error signal $(\psi(\beta))$ and Jacobian of softmax $(\partial \rvp(\beta)/\partial \bar{\rva}^K)$ \vs $\beta$ for a random correctly classified logits. Notice that, $\psi(\beta)$ and mean \acrshort{JSV} of $\partial \rvp(\beta)/\partial \bar{\rva}^K$ behave similarly.
    (b) Hessian norm vs. $\beta$ on a random correctly classified image. The plot clearly shows a concave behaviour and the maximum point occurs at a small value where both $\psi(\beta)$ and $\partial \rvp(\beta)/\partial \bar{\rva}^K$ are non zero.
    }
\label{fig:beta_vs_jacpsi_hess}
\end{figure*}
%
\begin{table}
\centering
\small
\begin{tabular}{l|l|ccc}
\toprule
                    & \textbf{Original}     & \textbf{Heuristic} & \textbf{\acrshort{NJS}} & \textbf{\acrshort{HNS}}  \\
\midrule
\textbf{\acrshort{BNN-WQ}}  & 0.8585 & 0.8845 & 0.4139 & \textbf{0.3450} \\
\textbf{\acrshort{BNN-WAQ}} & 0.7239 & 3.1578 & 0.3120 & \textbf{0.2774} \\
\bottomrule
\end{tabular}
\vspace{1ex}
\caption{\em \acrshort{CLEVER} Scores~\citenew{weng2018evaluating} for \acrshort{BNN-WQ} and \acrshort{BNN-WAQ} trained on \cifar{-10} using \sresnet{-18}. We compare CLEVER Scores returned for $L_1$ norm perturbation using different ways of temperature scaling applied. Here, Original refers to original network without temperature scaling and Heuristic denotes temperature scale with small $\beta=0.01$.    
}
\label{tab:clever}
\end{table}

\paragraph{\acrshort{CLEVER} Scores.} Recently \acrshort{CLEVER} Scores~\citenew{weng2018evaluating} have been proposed as an empirical estimate to measure robustness lower bounds for deep networks. It has been later shown that gradient masking issues cause \acrshort{CLEVER} to overestimate the robustness bounds~\citenew{goodfellow2018gradient}. Here we try to improve the \acrshort{CLEVER} scores using different ways of choosing $\beta$ in temperature scaling. For this experiment, we use \acrshort{CLEVER} implementation of Adversarial Training Toolbox\footnote{https://github.com/Trusted-AI/adversarial-robustness-toolbox}~\citenew{nicolae2018adversarial}. We set number of batches to 50, batch size to 10, radius to 5, and chose $L_1$ norm as hyperparameters (based on the \cite{weng2018evaluating}). We compare our variants namely \acrshort{NJS} and \acrshort{HNS} against heuristic choice of small $\beta=0.01$ and original \acrshort{CLEVER} Scores for \acrshort{BNN-WQ} and \acrshort{BNN-WAQ} (trained on \cifar{-10 using \sresnet{-18}}) in \tabref{tab:clever}. It can be clearly seen that our proposed variants improve the robustness bounds computed using \acrshort{CLEVER} whereas a heuristic choice of $\beta=0.01$ performs even worse. 

\subsection{C.4 ~Stability of \acrshort{PGD++} with \acrshort{NJS} with variations in $\rho$ }
We perform ablation studies with varying $\rho$ for \acrshort{PGD++} with \acrshort{NJS} in \tabref{tab:rho_vs_NJS} for \cifar{-10} dataset using \sresnet{-18} architecture. It clearly illustrates that our \acrshort{NJS} variant is quite robust to the choice of $\rho$ as we are able to achieve near perfect success rate with \acrshort{PGD++} with different values of $\rho$. As long as value of $\rho$ is large enough to avoid one-hot encoding on softmax outputs (in turn avoid $\|\psi(\beta)\|$ to be zero) of correctly classified sample, our approach with \acrshort{NJS} variant is quite stable.
\begin{table}[t]
\small
\centering
\begin{tabular}{l|c|c|c|c|c|c}
\toprule
\multirow{2}{*}{\textbf{Methods}} & \multicolumn{6}{c}{\textbf{\acrshort{PGD++} (\acrshort{NJS})} - \textbf{Varying} $\rho$}\\
& $1e-05$ & $1e-04$ & $1e-03$ & $1e-02$ & $1e-01$ & $2e-01$ \\
\midrule
\textbf{\acrshort{REF}} & 0.00 & 0.00 & 0.00 & 0.00 & 0.00 & 0.00 \\
\textbf{\acrshort{BNN-WQ}} & 0.00 & 0.00 & 0.00 & 0.00 & 0.00 & 0.00 \\
\textbf{\acrshort{BNN-WAQ}} & 0.15 & 0.08 & 0.04 & 0.03 & 0.04 & 0.02 \\
\bottomrule
\end{tabular}
\vspace{1ex}

\caption{\em Adversarial accuracy on the test set for binary neural networks using $L_\infty$ bounded \acrshort{PGD++} attack using \acrshort{NJS} with varying $\rho$. For different values of $\rho$, our approach is quite stable.
}
\label{tab:rho_vs_NJS}
\end{table}

\subsection{C.5 ~Signal Propagation and Input Gradient Analysis using \acrshort{NJS} and \acrshort{HNS}}
We first provide plots to to show how jacobian, error signal and hessian norm are influenced by $\beta$ in \figref{fig:beta_vs_jacpsi_hess}.
We then provide an example illustration in \figref{fig:beta_vs_ingrad} to better understand how the input gradient norm \ie, $\|\partial \ell(\beta)/\partial \rvx^0\|_2$, and norm of sign of input gradient, \ie, $\|\text{sign}(\partial \ell(\beta)/\partial \rvx^0)\|_2$ is influenced by $\beta$. It clearly shows that both the plots have a concave behavior where an optimal $\beta$ can maximize the input gradient. Also, it can be quite evidently seen in \figref{fig:beta_vs_ingrad} (b) that within an optimal range of $\beta$, gradient vanishing issue can be avoided. If $\beta\to 0$ or $\beta\to \infty$, it changes all the values in input gradient matrix to zero and inturn $\|\text{sign}(\partial \ell(\beta)/\partial \rvx^0)\|_2=0$. 
\begin{figure*}[t]
        \begin{subfigure}{0.48\textwidth}
          \includegraphics[width=\textwidth]{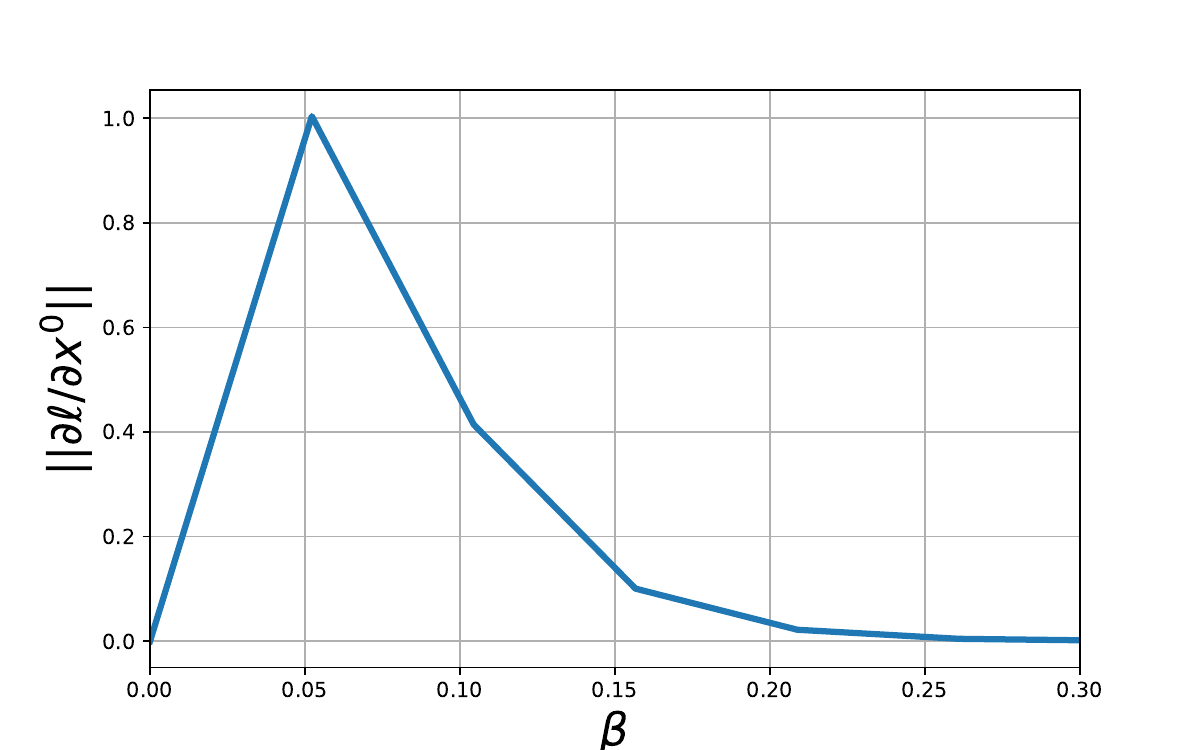}
          \vspace{-2ex}
        \caption{}          
        \end{subfigure}%
        \begin{subfigure}{0.48\textwidth}
          \includegraphics[width=\textwidth]{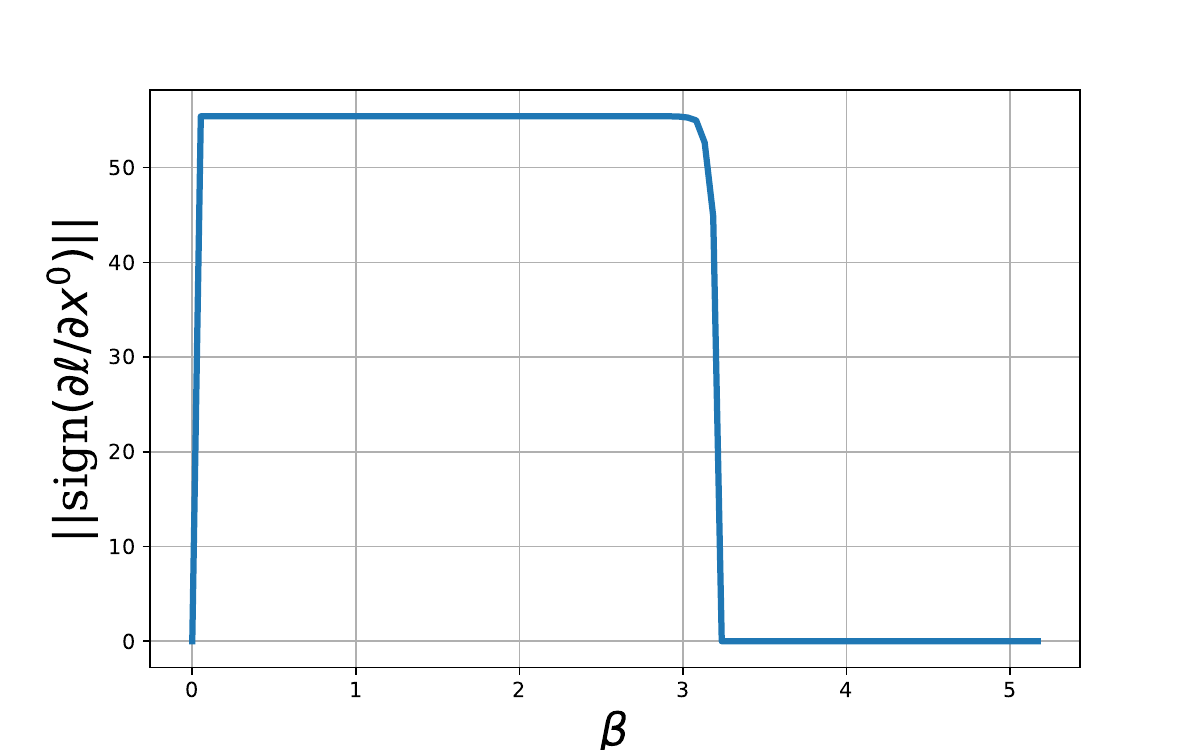}
          \vspace{-2ex}
        \caption{}          
        \end{subfigure}       
    \caption{\em Plots to show how variation in $\beta$ affects (a) norm of input gradient, \ie, $\|\partial \ell(\beta)/\partial \rvx^0\|_2$, (b) norm of sign of input gradient, \ie, $\|\text{sign}(\partial \ell(\beta)/\partial \rvx^0)\|_2$ on a random correctly classified image. Notice that, both input gradient and signed input gradient norm behave similarly, showing a concave behaviour. This plot is computed for \acrshort{BNN-WQ} network on \cifar{-10}, \sresnet{-18}. (b) clearly illustrates how optimum $\beta$ can avoid vanishing gradient issue since $\|\mbox{sign}(\partial \ell(\beta)/\partial \rvx^0)\|_2$ will only be zero if input gradient matrix has only zeros.
    }
\label{fig:beta_vs_ingrad}
\end{figure*}

\begin{table}[t]
\small
\centering
\begin{tabular}{ll|cccc}
\toprule
\multicolumn{2}{c|}{\textbf{Methods}} & \textbf{\acrshort{REF}} & \textbf{Adv. Train} & \textbf{\acrshort{BNN-WQ}} & \textbf{\acrshort{BNN-WAQ}} \\ 
\midrule
 \multirow{3}{*}{\textbf{\acrshort{JSV} (Mean)}} & \textbf{Orig.} & $8.09\mathrm{e}{+00}$ & $5.15\mathrm{e}{-01}$ & $3.53\mathrm{e}{+01}$ & $1.11\mathrm{e}{+00}$ \\ 
 & \textbf{\acrshort{NJS}} & $9.51\mathrm{e}{-01}$ & $5.70\mathrm{e}{-01}$ & $9.95\mathrm{e}{-01}$ & $2.24\mathrm{e}{-01}$ \\ 
 & \textbf{\acrshort{HNS}} & $2.38\mathrm{e}{+00}$ & $6.11\mathrm{e}{+00}$ & $1.19\mathrm{e}{+01}$ & $4.65\mathrm{e}{+00}$ \\ 

\midrule
 
 \multirow{3}{*}{\textbf{\acrshort{JSV} (Std.)}} & \textbf{Orig.} & $6.27\mathrm{e}{+00}$ & $4.10\mathrm{e}{-01}$ & $3.53\mathrm{e}{+01}$ & $1.97\mathrm{e}{+00}$ \\
 & \textbf{\acrshort{NJS}} & $7.58\mathrm{e}{-01}$ & $6.34\mathrm{e}{-01}$ & $9.71\mathrm{e}{-01}$ & $6.73\mathrm{e}{-01}$ \\
 & \textbf{\acrshort{HNS}} & $4.41\mathrm{e}{+00}$ & $5.34\mathrm{e}{+02}$  & $2.13\mathrm{e}{+02}$ & $1.24\mathrm{e}{+02}$ \\

\midrule

 \multirow{3}{*}{$\pmb{\|\psi\|_2}$} & \textbf{Orig.} & $9.08\mathrm{e}{-03}$ & $2.33\mathrm{e}{-01}$ & $6.20\mathrm{e}{-03}$ & $9.46\mathrm{e}{-03}$ \\
 & \textbf{\acrshort{NJS}} & $4.66\mathrm{e}{-01}$ & $2.35\mathrm{e}{-01}$ & $5.37\mathrm{e}{-01}$ & $1.20\mathrm{e}{-01}$ \\
 & \textbf{\acrshort{HNS}} & $1.48\mathrm{e}{-01}$ & $2.57\mathrm{e}{-01}$ & $2.07\mathrm{e}{-01}$ & $2.44\mathrm{e}{-01}$ \\

\midrule

 \multirow{3}{*}{$\pmb{\|\partial \ell/\partial \rvx^0\|_2}$} & \textbf{Orig.} & $2.42\mathrm{e}{-01}$ & $8.52\mathrm{e}{-02}$ & $2.27\mathrm{e}{-01}$ & $6.33\mathrm{e}{-02}$ \\
 & \textbf{\acrshort{NJS}} & $9.52\mathrm{e}{-01}$ & $1.10\mathrm{e}{-01}$ & $8.91\mathrm{e}{-01}$ & $1.24\mathrm{e}{-01}$ \\
 & \textbf{\acrshort{HNS}} & $7.49\mathrm{e}{-01}$ & $8.18\mathrm{e}{-01}$ & $3.70\mathrm{e}{-01}$ & $2.70\mathrm{e}{-01}$ \\

\midrule

 \multirow{3}{*}{$\pmb{\|\mbox{sign}\big(\frac{\partial \ell}{\partial \rvx^0}\big)\|_2}$} & \textbf{Orig.} &  $5.55\mathrm{e}{+01}$ & $5.54\mathrm{e}{+01}$ & $4.39\mathrm{e}{+01}$ & $5.55\mathrm{e}{+01}$ \\
 & \textbf{\acrshort{NJS}} & $5.55\mathrm{e}{+01}$ & $5.54\mathrm{e}{+01}$ & $5.55\mathrm{e}{+01}$ & $5.55\mathrm{e}{+01}$ \\
 & \textbf{\acrshort{HNS}} & $5.55\mathrm{e}{+01}$ & $5.54\mathrm{e}{+01}$ & $5.55\mathrm{e}{+01}$ & $5.55\mathrm{e}{+01}$ \\

\bottomrule
\end{tabular}
\vspace{1ex}

\caption{\em Mean and standard deviation of \acrfull{JSV}, mean $\|\psi\|_2$, mean $\|\partial \ell/\partial \rvx^0\|_2$ and mean $\|\mbox{sign}(\partial \ell/\partial \rvx^0)\|_2$ for different methods on \cifar{-10} with \sresnet{-18} computed with 500 correctly classified samples. Note here for \acrshort{NJS} and \acrshort{HNS}, \acrshort{JSV} is computed for scaled jacobian \ie $\beta \rmJ$. Also note that, values of $\|\psi\|_2$, $\|\partial \ell(\beta)/\partial \rvx^0\|_2$ and $\|\text{sign}(\partial \ell(\beta)/\partial \rvx^0)\|_2$ are larger for our \acrshort{NJS} and \acrshort{HNS} variant (for most of the networks) as compared with network with no $\beta$, which clearly indicates better gradients for performing gradient based attacks.}


\label{tab:psinorm_jsv_ingrad_study}
\end{table}

We also provide the signal propagation properties as well as analysis on input gradient norm before and after using the $\beta$ estimated based on \acrshort{NJS} and \acrshort{HNS} in \tabref{tab:psinorm_jsv_ingrad_study}. For binarized networks as well floating point networks tested on \cifar{-10} dataset using \sresnet{-18} architecture, our \acrshort{HNS} and \acrshort{NJS} variants result in larger values for $\|\psi\|_2$, $\|\partial \ell(\beta)/\partial \rvx^0\|_2$ and $\|\text{sign}(\partial \ell(\beta)/\partial \rvx^0)\|_2$. This reflects the efficacy of our method in overcoming the gradient vanishing issue. It can be also noted that our variants also improves the signal propagation of the networks by bringing the mean \acrshort{JSV} values closer to $1$. 


\subsection{C.6 ~Ablation for $\rho$ vs. \acrshort{PGD++} accuracy}
In this subsection, we provide the analysis on the effect of bounding the gradients of the network output of ground truth class $k$, \ie $\partial \ell(\beta) / \partial \bar{a}_k^K$. Here, we compute $\beta$ using Proposition \myref{1} for all correctly classified images such that $1-\softmax(\beta a_k^K)>\rho$ with different values of $\rho$ and report the \acrshort{PGD++} adversarial accuracy in \tabref{tab:rho_vs_fgsm++}. It can be observed that there is an optimum value of $\rho$ at which \acrshort{PGD++} success rate is maximized, especially on the adversarially trained models. This can also be seen in connection with the non-linearity of the network where at an optimum value of $\beta$, even for robust (locally linear) \citenew{moosavi2019robustness,qin2019adversarial} networks such as adversarially trained models, non-linearity can be maximized and better success rate for gradient based attacks can be achieved. Our \acrshort{HNS} variant essentially tries to achieve the same objective while trying to estimate $\beta$ for each example.
\begin{table}[t]
\small
\centering
\begin{tabular}{l|cccccc}
\toprule
\multirow{2}{*}{\textbf{Methods}} & \multicolumn{6}{c}{\textbf{\acrshort{PGD++} with Varying} $\rho$}\\ \cmidrule{2-7}

& $1e-15$ & $1e-09$ & $1e-05$ & $1e-01$ & $2e-01$ & $5e-01$ \\
\midrule
\textbf{\acrshort{REF}} & 0.00 & 0.00 & 0.00 & 0.00 & 0.00 & 0.00 \\
\textbf{\acrshort{BNN-WQ}} & 9.61 & 0.04 & 0.00 & 0.00 & 0.00 & 0.00 \\
 \midrule
\textbf{\acrshort{REF}}$^*$ & 48.18 & \textbf{47.66} & 48.00 & 53.09 & 54.58 & 57.57 \\
\textbf{\acrshort{BNN-WQ}}$^*$ & 40.66 & \textbf{40.01} & 40.04 & 45.09 & 46.57 & 49.72\\
\bottomrule
\end{tabular}
\vspace{1ex}
\caption{\em Adversarial accuracy on the test set for adversarially trained networks and binary neural networks using $L_\infty$ bounded \acrshort{PGD++} attack with varying $\rho$ as lower bound on the gradient of network output for ground truth class $k$. Here * denotes the adversarially trained models obtained where adversarial samples are generated using $L_\infty$ bounded \acrshort{PGD} attack with with $T = 7$ iterations, $\eta=2$ and $\epsilon=8$. Note, here \acrshort{PGD++} attack refers to \acrshort{PGD} attack where $\partial \ell(\beta) / \partial \bar{a}_k^K$ is bounded by $\rho$ for each sample, where $k$ is ground truth class.
}
\label{tab:rho_vs_fgsm++}
\end{table}

\section{Limitations and Impact}
The proposed \acrshort{FGSM++}/\acrshort{PGD++} variants can only improve the adversarial success rate on the networks suffering from poor signal propagation. To achieve the objective of improving signal propagation to estimate better gradient signals for an adversarial attack, we proposed a temperature scaling based technique. If the Jacobian $\rmJ$ has zero singular values, our temperature scaling technique has no effect in those dimensions, although we find this issue does not arise in practise where our attack variants achieve near complete success rate on \bnns{} in all our experiments. We would like to also point out that although we believe that our proposed attack variants would demystify the fake robustness of multi-bit quantized networks with assumption of poor signal propagation, our paper restricts the study to \bnns{} in all our experiments.

Research that explores the security and safety aspect of machine learning models (for e.g. adversarial attacks on deep neural networks) need to move with caution. Our paper explores improvement of existing adversarial attack to reveal the potential vulnerability of \bnns{}. Similar to most of the existing adversarial attack literature, this could potentially result in breaking many existing machine learning systems. But we also believe that this revelation can encourage researchers to design more robust defense mechanisms that might help the community to further making deployment of  machine learning models more secure and safe. Thus, the positive gains of this research outweighs any potential harmful impacts.

\section{Acknowledgements}
This work was supported by the Australian Research Council Centre of Excellence for Robotic Vision (project number CE140100016). We acknowledge the DATA61, CSIRO for their support and thank Puneet Dokania for useful discussions.

\bibliography{iga-aaai}

\end{document}